\documentclass{article}

\PassOptionsToPackage{numbers,compress}{natbib}
\usepackage[main, final]{neurips_2026}
\usepackage{microtype}
\usepackage{booktabs} 

\usepackage{mathtools}
\usepackage[utf8]{inputenc}
\usepackage[T1]{fontenc}
\usepackage{url}
\usepackage{graphicx,wrapfig,lipsum}
\usepackage{amsfonts}
\usepackage{nicefrac}
\usepackage{microtype}
\usepackage{dsfont}
\usepackage{subcaption}
\usepackage{MnSymbol}
\usepackage[noend]{algpseudocode}
\usepackage{algorithm}
\usepackage{booktabs,tabularx, multirow}
\usepackage{tikz}
\usetikzlibrary{bayesnet}
\usepackage{mathtools}
\usepackage{amsmath,amsthm}
\usepackage{soul}
\usepackage{blindtext}
\usepackage{tcolorbox}
\usepackage{todonotes}
\usepackage{makecell}
\usepackage{thmtools}
\usepackage{wrapfig}
\newtheorem{theorem}{Theorem}

\DeclarePairedDelimiterX{\infdivx}[2]{(}{)}{%
  #1\;\delimsize\|\;#2%
}
\DeclarePairedDelimiterX{\inftvx}[2]{(}{)}{%
  #1, #2%
}

\DeclareMathOperator*{\argmin}{arg\,min}

\newtheorem{prop}{Proposition}
\newtheorem{definition}{Definition}

\newtheorem{lemma}{Lemma}

\usepackage{pifont}

\usepackage{xcolor, color, colortbl}
\usepackage{listings}
\usepackage{rotating}
\usepackage{nccmath} 
\definecolor{mydarkblue}{rgb}{0,0.08,0.45}
\definecolor{myblue}{HTML}{D2E4FC}

\definecolor{mygreen}{HTML}{008000}

\definecolor{myorange}{HTML}{FFC48C}

\definecolor{Gray}{gray}{0.92}

\definecolor{codegreen}{rgb}{0,0.6,0}
\definecolor{codegray}{rgb}{0.5,0.5,0.5}
\definecolor{codepurple}{rgb}{0.58,0,0.82}
\definecolor{backcolour}{rgb}{0.95,0.95,0.92}
 
\lstdefinestyle{mystyle}{
    backgroundcolor=\color{backcolour},   
    commentstyle=\color{codegreen},
    keywordstyle=\color{magenta},
    numberstyle=\tiny\color{codegray},
    stringstyle=\color{codepurple},
    basicstyle=\ttfamily\footnotesize,
    breakatwhitespace=false,         
    breaklines=true,                 
    captionpos=b,
    keepspaces=true,
    numbers=left,
    numbersep=5pt,
    showspaces=false,
    showstringspaces=false,
    showtabs=false,                  
    tabsize=2
}
 
\definecolor{shadecolor}{RGB}{242, 242, 242}

\definecolor{mydarkblue}{rgb}{0,0.08,0.45}
\usepackage[colorlinks=true,
    linkcolor=mydarkblue,
    citecolor=mydarkblue,
    filecolor=mydarkblue,
    urlcolor=mydarkblue]{hyperref}

\usepackage{fancyvrb}

\title{Active Learning of Conditional Generative Models \\via the Transport Neural Tangent Kernel}

\author{%
  Jayoung Ryu$^{1}$ \quad Kyunghyun Cho$^{1}$ \quad Romain Lopez$^{1,2}$ \\
  \\
  $^{1}$Department of Computer Science, Courant Institute School of Mathematics, \\Computing, and Data Science, New York University\\
  $^{2}$Department of Biology, New York University\\
  \\
  \texttt{\{jayoung.ryu, kyunghyun.cho, romain.lopez\}@nyu.edu}
}

\begin{document}

\maketitle

\begin{abstract}
Many high-throughput scientific experiments produce distributional outputs, requiring conditional generative models to interpret the data, and to predict outcomes at new inputs. Acquiring experimental data at new conditions (e.g., for a new antigen or a new cell type) is often expensive, motivating the need for automated experimental design methods. However, existing active learning frameworks are limited as they assume scalar or vector-valued outputs. We first show that the Wasserstein test risk of a conditional generative task decomposes into a mean term and a shape term, demonstrating why mean-based active learning is blind to distributional shape. To capture both, we propose the \emph{transport Neural Tangent Kernel} (tNTK), a tractable kernel that measures how sensitive a generative model's predicted distribution is with respect to its parameters. The tNTK recovers the standard NTK in the deterministic limit and upper-bounds the posterior Fréchet variance of the prediction under small parameter perturbations. It admits a closed-form expression for affine Gaussian transport and tractable Monte Carlo estimators otherwise. Treating the tNTK as the covariance function of a Gaussian process, we greedily select inputs that minimize the total GP posterior variance across the pool, yielding a distribution-aware active learning strategy. We evaluate this strategy across four classes of conditional generative models, including conditional flow matching and conditional Monge gap, on synthetic distributions, intent-labeled sentences, and single-cell perturbation responses. In these experiments, the tNTK outperforms the baseline NTK computed from the conditional mean, demonstrating the value of distribution-aware active learning for generative-modeling tasks.

\end{abstract}

\section{Introduction}

Decades after a logic-based ``robot scientist'' first autonomously characterized gene functions in yeast metabolism~\citep{king2004functional, king2009automation}, machine learning has become a standard tool for the automated design of scientific experiments. These approaches have demonstrated impact across a wide range of scientific fields, including inorganic materials synthesis~\citep{szymanski2023autonomous}, drug discovery~\citep{tosh2025bayesian, frey2025lab, yang2025active, jiang2024rapid}, and single-cell perturbation biology~\cite{rood2024toward, zhang2023active, Huang2024-kj,pmlr-v267-qin25g,demeo2025active}. They have also been applied to large-scale image classification~\citep{sener2018active} 
and to the alignment of large language models~\citep{mehtasample, muldrew2024active}.

We consider the setup of pool-based active learning (AL)~\citep{cohn1994improving, cohn1996active}, in which the acquisition strategy iteratively assembles a training set $\mathcal{X}_\text{train}$ from a set of candidate inputs $\mathcal{X}_\text{pool} \subseteq \mathcal{X}$. For each selected input $x \in \mathcal{X}_\text{train}$, the corresponding output $o \in \mathcal{O}$ is acquired using the experimental system. The goal of an AL strategy is to assemble a training set with a fixed number of data points so that a surrogate predictive model $f(x)$ will generalize best to unseen inputs (i.e., a test set) --- a problem extensively studied for classification~\cite{sener2018active,Settles2009-bv,ash2020deep,Houlsby2011-bp,Kirsch2019-qt,Pinsler2019BayesianBA} ($\mathcal{O} = \{1, \ldots, K\}$) and regression~\cite{sugiyama06a,reichart2008multi,Holzmuller2022-do,Ash2021-zv} (single- or multi-task, $\mathcal{O} \subseteq \mathbb{R}^d$). 

In many scientific applications, however, the experimental output is a probability distribution: $\mathcal{O} = \mathcal{P}(\mathcal{Y})$ for some underlying sample space $\mathcal{Y} \subseteq \mathbb{R}^d$, observed through a finite set of samples for each input. For example, chemical or genetic perturbation assays measure genomics data for hundreds of cells per perturbation, profiling entire distributions over cell states~\cite{rood2024toward}. Regulatory genomics assays capture hundreds of thousands of hypersensitive DNA sequences in the genome for each cell type~\cite{meuleman2020index}. In both cases, the output mean $\mathbb{E}[y \mid x]$ is a poor substitute for the entire output $p(y \mid x)$ as the distribution itself is of scientific interest. For this reason, existing AL paradigms are not suited to distributional outputs because they have no other option but to reduce the output to a pre-specified scalar functional (e.g., mean or an event probability~\citep{bect2012sequential, sapsis2020output}). 

Conditional generative models such as conditional optimal transport~\citep{Bunne2022-qx, Driessen2025-vw}, flow matching~\citep{Lipman2022-mt}, and diffusion models~\citep{Ho2020-hd} are compelling options for predictive modeling of distributions. They demonstrated impressive performance for image and text generation~\citep{Zhan2024-yv, Liang2024-zy}, as well as perturbation effect prediction~\cite{Driessen2025-vw, Bunne2023-jf, Klein2025-qb, ryu2025cross, rohbeck2025modeling} and conditional DNA sequence generation~\cite{DaSilva2026-eq}. Despite their adoption in these domains, no existing acquisition strategies directly leverage these generative models within an active learning framework for experimental design.  

To address this gap, we develop a framework for active learning of conditional generative models that operates natively on distributional outputs. Specifically, we introduce the transport Neural Tangent Kernel (tNTK), a tractable kernel that measures the sensitivity of the predicted distribution $p_\theta(y \mid x)$ to perturbations of the model parameters $\theta$. The tNTK admits two complementary interpretations: it upper-bounds the Fréchet posterior variance under a Laplace approximation to the parameter posterior, and it controls the metric derivative of the predicted distribution on the Wasserstein manifold. We characterize the geometric gap and show the bound is tight in dimension one. To propose the first distribution-aware active learning strategy, we treat the tNTK as a Gaussian Process (GP) covariance function, and greedily select inputs that minimize the total posterior variance across the pool, following existing kernel-based active learning frameworks~\citep{Holzmuller2022-do,Ash2021-zv}. 

After introducing the necessary background (Section~\ref{sec:background}), we decompose the risk of a conditional generative prediction task into a mean and a shape term, and illustrate on a tractable 1D example that mean-based acquisition is suboptimal (Section~\ref{sec:histogram}). We then introduce the transport Neural Tangent Kernel (tNTK) as well as our distribution-aware AL strategy (Section~\ref{sec:transport_kernel}). Finally, we conduct an empirical evaluation across four classes of conditional generative models --- affine Gaussian transport, conditional quantile prediction, conditional flow matching, and conditional Monge gap --- on simulated data where ground truth is available, as well as on real-world data from intent-labeled sentences and single-cell perturbation datasets (Section~\ref{sec:experiments}). Across 26 out of 30 settings, the tNTK outperforms baselines that operate on the conditional mean alone.

\section{Background}
\label{sec:background}
\subsection{Conditional Generative Models (CGMs)}
\label{sec:cgm}
We consider training data $\mathcal{D} = \{(x_i, \{y_{i,j}\}_{j=1}^{n_i})\}_{i=1}^{N}$, where each input $x_i \in \mathcal{X}_\text{train}$ is paired with a set of $n_i$ independent and identically distributed samples $\{y_{i,j}\}_{j=1}^{n_i}$ drawn from an unknown conditional distribution $\mu(x_i) \in \mathcal{P}(\mathcal{Y})$.  \emph{Transport-based} CGMs induce output distributions by pushing forward a fixed base measure $\nu \in \mathcal{P}(\mathcal{Y})$ through a parametric map $T_\theta(x): \mathcal{Y} \rightarrow \mathcal{Y}$:
\begin{align}
    \mu_\theta(x) = T_\theta(x)\# \nu,
    \label{eq:pushforward}
\end{align}
where $\theta \in \mathbb{R}^p$ denotes the model parameters. Concretely, a sample from $\mu_\theta(x)$ is obtained by drawing $y \sim \nu$ and returning $T_\theta(x)(y)$. This formulation covers various common families of conditional generative models, two of which we now briefly detail.

\paragraph{Conditional optimal transport} Conditional optimal transport frameworks learn a parametric map $T_\theta(x)$ that approximates the optimal transport map from $\nu$ to $\mu(x)$ across all input points $x \in \mathcal{X}_\text{train}$. CondOT~\citep{Bunne2022-qx} parameterizes $T_\theta(x)(y) = \nabla_y \Phi_\theta (x, y)$ where $\Phi_\theta$ is a partially input convex neural network (PICNN)~\citep{Amos2016-dq}, which ensures that $T_\theta(x)$ is a valid Monge map. Conditional Monge gap~\citep{Driessen2025-vw} drops the PICNN constraint and instead trains an unconstrained neural network with two terms: a distributional loss between $T_\theta(x)\# \nu$ and $\mu(x)$, and the Monge gap regularizer~\citep{Uscidda2023-as}, which penalizes the suboptimality of $T_\theta(x)$ as a transport map between $\nu$ and its own pushforward. This approach improves numerical stability and empirical performance over CondOT.

\paragraph{Conditional flow matching (CFM)} CFM~\citep{Lipman2022-mt}\footnote{We use ``conditional'' to refer to conditioning on an external covariate $x$ (e.g., ~\citep{Klein2025-qb}). This differs from the original usage in \citet{Lipman2022-mt}, where ``conditional'' refers to per-sample conditional probability paths used to train an unconditional vector field.} learns a parametric, time- and condition-dependent vector field $v_\theta^{t, x}: \mathcal{Y} \rightarrow \mathcal{Y}$ by regressing on a target field $u^{t, x}$ constructed from a chosen conditional probability path interpolating between $\nu$ at $t=0$ and $\mu(x)$ at $t=1$. The flow $\phi^{t, x}(y)$ is defined as the solution of the following ordinary differential equation:
\begin{equation}
    \frac{d}{dt}\phi^{t, x}(y) = v_\theta^{t, x}(\phi^{t, x}(y)), \quad \phi^{0, x}(y) = y. \label{eq:cfm_ode}
\end{equation}
At inference time, samples from $\mu_\theta(x) = (\phi^{1, x})\# \nu$ are obtained by sampling $y \sim \nu$ and integrating the ODE from $t=0$ to $t=1$ with initial condition $\phi^{0, x}(y) = y$. Although CFM specifically applies to the case where $\nu$ is a noise distribution, variants such as OT-CFM~\citep{Tong2023-mg} extend CFM to cases where $\nu$ is a data distribution using optimal transport couplings.

\subsection{Active Learning for Regression}
\label{sec:al}

In pool-based AL, the acquisition strategy has access to a candidate set $\mathcal{X}_{\text{pool}} \subseteq \mathcal{X}$ of unlabeled inputs and selects a subset $\mathcal{S} \subseteq \mathcal{X}_{\text{pool}}$ for which outputs are measured with the experimental system. A key heuristic in the field is to select the input that minimizes the total posterior variance of the surrogate model $f_\theta$ on the pool,
\begin{equation}
    x^* = \argmin_{x' \in \mathcal{X}_{\text{pool}}} \sum_{x \in \mathcal{X}_{\text{pool}}} \operatorname{Var}\!\left[f_\theta(x) \mid \mathcal{X}_{\text{train}} \cup \{x'\}\right].
    \label{eq:total_variance}
\end{equation}
Such candidates are often the most informative for generalization to a test set~\cite{Ash2021-zv}. Active learning strategies for regression are further discussed in~\citep{Holzmuller2022-do}.

For neural network-based models, predictive uncertainty is intractable in general and must be approximated. Common approaches include approximate Bayesian inference over the model parameters~\citep{rakesh2021efficacy}, Monte-Carlo dropout~\citep{pmlr-v48-gal16}, and kernel-based approximations derived from the finite-width Neural Tangent Kernel (NTK)~\citep{Khan2019-qd, novak2022fast}:
\begin{equation}
    k_\theta^{\mathrm{NTK}}(x, x') = \nabla_\theta f(x)^\top \nabla_\theta f(x').
    \label{eq:ntk}
\end{equation}
The predictive distribution $f$ after linearization around the trained parameters $\theta$ is well approximated by a Gaussian process with covariance $k_\theta^{\mathrm{NTK}}$. Informative candidates can then be selected using standard kernel-based acquisition strategies~\citep{Holzmuller2022-do},  such as the total-posterior-variance minimization described above. All of these methods assume that $f(x)$ is a scalar or low-dimensional prediction, and do not extend to settings where the model's output is a probability distribution.

\section{Distributional Risk Decomposition}
\label{sec:histogram}

We illustrate the need for distribution-aware AL methods by analyzing the test risk of conditional generative models under the 2-Wasserstein loss. These derivations suggest that existing AL methods based on test risk minimization for conditional mean prediction are suboptimal. We then present a fully-tractable one-dimensional example alongside a numerical experiment. 

\subsection{Mean-shape decomposition of distributional risk}

For the following theoretical derivations, we assume that distribution $\mu(x_i)$ (or equivalently, infinite samples from it) is observed for each input $x_i$. This allows us to define our training data as $\mathcal{D} = \{(x_i, \mu(x_i))\}_{i=1}^N$. We also assume the availability of a posterior distribution $\theta \mid \mathcal{D}$ over the parameters $\theta$ of a CGM that reflects uncertainty in the predicted distribution $\mu_\theta(x) = T_\theta(x)_\# \nu$. For a test point $x \in \mathcal{X}$, the risk under the squared Wasserstein loss is defined as:
\begin{align}
    \mathcal{R}(x, \mathcal{D}) = \mathbb{E}_{\theta \mid \mathcal{D}}\!\left[W_2^2(\mu_\theta(x), \mu(x))\right], \label{eq:test_w2_risk}
\end{align}
and admits the following decomposition:
\begin{restatable}[Mean-shape decomposition of risk]{prop}{riskdecomposition}
\label{prop:risk_decomposition}
The Wasserstein risk defined in Equation~\eqref{eq:test_w2_risk} decomposes as
\begin{equation}
\mathcal{R}(x, \mathcal{D}) = \underbrace{ \mathbb{E}_{\theta \mid \mathcal{D}}\!\left[\|m_\theta(x) - m(x)\|^2\right]}_{\mathcal{R}_\text{mean}(x, \mathcal{D})} \;+\; \underbrace{\mathbb{E}_{\theta \mid \mathcal{D}}\!\left[W_2^2(\mu^c_\theta(x), \mu^c(x))\right]}_{\mathcal{R}_\text{shape}(x, \mathcal{D})}, \label{eq:risk_decomp}
\end{equation}
where $m_\rho = \int y\, d\rho(y)$ denotes the mean of $\rho$, and $\rho^c$ is obtained by centering $\rho$ for any $\rho \in \mathcal{P}_2(\mathcal{Y})$.
\end{restatable}
The proof is given in Appendix~\ref{app:proof_risk_decomposition}. Classical regression-based active learning traditionally aims at (approximately) minimizing the mean-risk $\mathbb{E}_{x \sim p_\text{test}}\mathcal{R}_\text{mean}(x, \mathcal{D})$ ~\citep{sugiyama06a}. Because they do not consider the additional shape-related term, these methods may perform poorly whenever shape matters.

\subsection{One-dimensional example}
\label{sec:histogram_1d}
The one-dimensional case is a remarkable instance since Wasserstein distances admit closed-form expressions. Let $F^{-1}(x; \tau)$ (resp. $F^{-1}_\theta(x; \tau)$) denote the $\tau$-th quantile of distribution $\mu(x)$ (resp. $\mu_\theta(x)$) for $\tau \in [0, 1]$. We have the following identities:
\begin{restatable}[1D posterior risk]{lemma}{quantilevariance}
\label{lem:quantile_variance}
Assume the model is unbiased at the mean and every quantile level, i.e., $\mathbb{E}_{\theta \mid \mathcal{D}}[m_\theta(x)] = m(x)$ and $\mathbb{E}_{\theta \mid \mathcal{D}}[F^{-1}_\theta(x; \tau)] = F^{-1}(x; \tau)$ for all $x, \tau$. Then
\begin{equation}
\mathcal{R}_\text{mean}(x, \mathcal{D}) = \operatorname{Var}_{\theta \mid \mathcal{D}}[m_\theta(x)], \qquad
\mathcal{R}_\text{shape}(x, \mathcal{D}) = \int_0^1 \operatorname{Var}_{\theta \mid \mathcal{D}}[g_\theta(x; \tau)]\, d\tau, \label{eq:risk_1d_var}
\end{equation}
where $g_\theta(x; \tau) = F^{-1}_\theta(x; \tau) - m_\theta(x)$ denotes the centered quantile function of $\mu_\theta(x)$.
\end{restatable}
The proof appears in Appendix~\ref{app:proof_risk_1d}. The shape-related term becomes tractable as it reduces to an integral of marginal posterior variances over quantile levels $\tau \in [0, 1]$. 

We numerically illustrate this on simulated data (details in Appendix~\ref{app:condhistdetails}). Briefly, $\mu(x)$ is a two-component Gaussian mixture model with mean and variance depending on $x \in [0,1]$ (Figure~\ref{fig:toyhistogram}A). $m(x)$ and  $g(x; \tau)$ are modeled for a discrete grid of $\tau$ levels using independent Gaussian processes. This approach provides direct access to the posterior variances in \eqref{eq:risk_1d_var}. We run AL campaigns to compare three acquisition strategies: Wasserstein AL, which minimizes the expected total risk, mean AL, which minimizes only its mean component,
\begin{align}
    x^*_{\text{W}} &= \argmin_{x' \in \mathcal{X}_{\text{pool}}} \mathbb{E}_{x \sim p_{\text{test}}} \mathcal{R}(x, \mathcal{D}_{\text{train}} \cup \{x'\}), \\
    x^*_{\text{mean}} &= \argmin_{x' \in \mathcal{X}_{\text{pool}}} \mathbb{E}_{x \sim p_{\text{test}}} \mathcal{R}_{\text{mean}}(x, \mathcal{D}_{\text{train}} \cup \{x'\}),
\end{align}
and random sampling.


Wasserstein AL outperforms Mean AL by over an order of magnitude (Figure~\ref{fig:toyhistogram}B). Mean AL still improves over random in terms of mean error (Figure~\ref{fig:toyhistogram}C) but plateaus on shape error (Figure~\ref{fig:toyhistogram}D). These results illustrate that ignoring distribution shape for active learning is suboptimal. 

\begin{figure}[t]
  \centering
  \includegraphics[width=\textwidth]{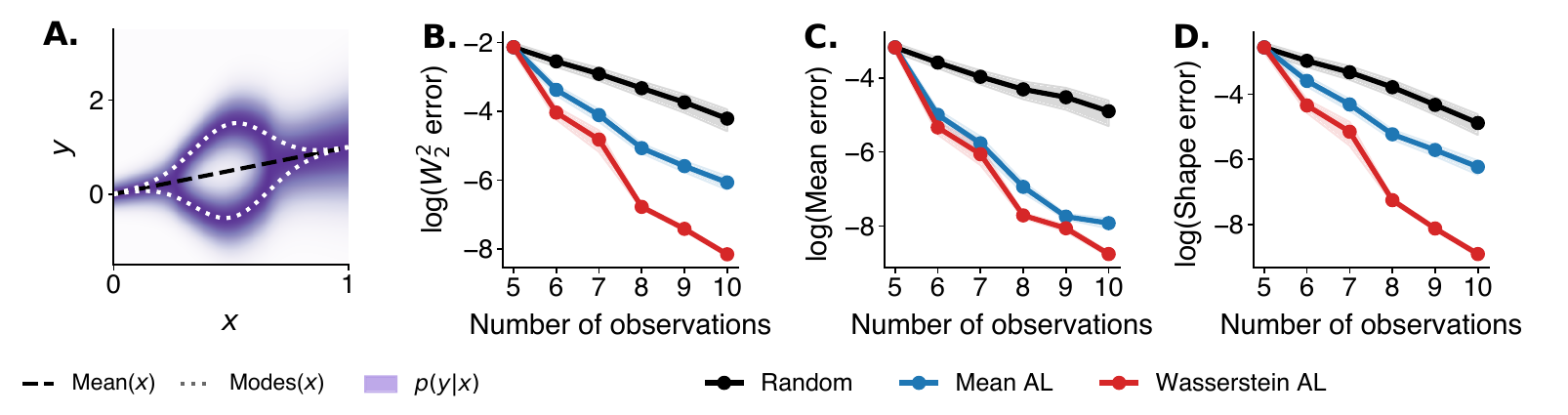}
  \vspace{-0.5cm}
  \caption{Synthetic conditional density estimation. (A) Ground truth conditional density $\mu(x)$. (B--D) Test error after active learning as a function of training set size: total risk (B), mean prediction error (C), and shape error (D). Results averaged over $50$ random initializations; shaded bands show $\pm 2$ standard errors.}
  \vspace{-0.3cm}
  \label{fig:toyhistogram}
\end{figure}

\section{Transport Kernels for Active Learning of Conditional Distributions}
\label{sec:transport_kernel}
For most applications of interest, neither the test risk nor the shape term from~\eqref{eq:risk_decomp} is tractable. Most fundamentally, the ground-truth distribution $\mu(x)$ is unknown at unseen inputs $\mathcal{X}_\text{pool}$. Even if it were known, the 2-Wasserstein distance remains intractable for neural transport maps. Those represent significant challenges for distribution-aware active learning.

We circumvent these issues by considering the posterior Fréchet variance as a measure of uncertainty in a generative model's predictions $\mu_\theta(x)$. We bound this quantity by linearizing the transport map $T_\theta(x)$ around its parameters $\theta$, which yields the transport Neural Tangent Kernel (tNTK). We briefly discuss properties of the tNTK, how to estimate it efficiently, and then use it to design a distribution-aware active learning strategy.
 
\subsection{Quantifying model uncertainty via parameter sensitivity}
We consider the \emph{posterior Fréchet variance},
\begin{align}
    V_F(x) := \mathrm{Var}_{W_2}\!\left[\mu_\theta(x) \mid \mathcal{D}\right] =  \inf_{\rho \in \mathcal{P}_2(\mathcal{Y})} \mathbb{E}_{\theta \mid \mathcal{D}}\!\left[W_2^2(\mu_\theta(x), \rho)\right], \label{eq:frechet_variance}
\end{align}
defined as the minimum mean squared Wasserstein distance from a fixed reference measure to a posterior draw $\mu_\theta(x)$~\citep{frechet1948elements}. The minimizer is the Wasserstein barycenter $\bar{\mu}_\theta(x)$ of the posterior law of $\mu_\theta(x)$~\citep{agueh2011barycenters}. The posterior Fréchet variance $V_F(x)$ has two properties that make it especially relevant for active learning. First, it lower bounds the risk $V_F(x) \leq \mathcal{R}(x, \mathcal{D})$, with equality when the Wasserstein barycenter equals the data-generating distribution $\bar{\mu_\theta}(x) = \mu(x)$. Second, in the deterministic limit, where $\mu_\theta(x)$ collapses to a Dirac mass at $f_\theta(x)$, $V_F(x)$ recovers the standard posterior variance $\mathrm{Var}_{\theta \mid \mathcal{D}}[f_\theta(x)]$. 

Although $V_F(x)$ is intractable in general, since the Wasserstein barycenter of the posterior law of $\mu_\theta(x)$ has no closed form, we show that linearizing the transport map yields a tractable upper bound.

\begin{restatable}[Upper bound on the posterior Fréchet variance]{theorem}{sensitivitybound}
\label{thm:sensitivitybound}
Let $\theta \mid \mathcal{D} \sim \mathcal{N}(\hat\theta, \Sigma)$, with $\Sigma \in \mathbb{R}^{p \times p}$ positive definite and $\|\Sigma\|_{\mathrm{op}} < 1$. Assume that for every $x \in \mathcal{X}$ and for $\nu$-almost every $y \in \mathcal{Y}$, the map $\theta \mapsto T_\theta(x)(y)$ is $C^2(\mathbb{R}^p)$, and that the Jacobian $D_\theta T_{\hat\theta}(x)(\cdot)$ and the pointwise supremum of the operator norm of the Hessian, $\sup_{\theta} \|D_\theta^2 T_\theta(x)(\cdot)\|_{\mathrm{op}}$, are in $L^2(\nu)$. Then, there exists $C(x)$ such that:
\begin{equation}
    V_F(x) \leq \operatorname{Tr}\!\left(G_{\hat\theta}(x)\, \Sigma\right) + C(x)\|\Sigma\|_{\mathrm{op}}^{3/2},\label{eq:sensitivity_bound}
\end{equation}
where the $p \times p$ parameter sensitivity matrix at input $x$ is
\begin{equation}
    G_{\hat\theta}(x) = \int D_\theta T_{\hat\theta}(x)(y)^\top D_\theta T_{\hat\theta}(x)(y)\, d\nu(y). \label{eq:G_def}
\end{equation}
Inequality \eqref{eq:sensitivity_bound} is tight to leading order when $\dim(\mathcal{Y}) = 1$ and $T_\theta(x)$ is monotone in $y$.
\end{restatable}
The proof is given in Appendix~\ref{app:proof_sensitivity_bound}. Theorem~\ref{thm:sensitivitybound} identifies that under a Laplace approximation~\citep{mackay1992practical, daxberger2021laplace} to the parameter posterior,  $\operatorname{Tr}\!\left(G_{\hat\theta}(x)\, \Sigma\right)$ plays the role of a tractable measure of the predictive variance in the generative model's predictions, in the limit of small variance. This is important given that exact computations of the Wasserstein metric tensor are otherwise intractable for neural transport maps~\citep{li2023wasserstein, liu2022neural}. 

The tightness of the bound from \eqref{eq:sensitivity_bound} is an important consideration that we investigate in several ways. First, we give a geometric interpretation of $G_\theta(x)$ by upper-bounding the Wasserstein metric derivative of trajectories of $\mu_\theta(x)$ for small deviations around $\theta$ and characterize when the bound is tight (Appendix~\ref{app:geometry}). Second, for affine transport of Gaussian distributions (Appendix~\ref{app:numerical_gap_analysis}), we find numerically that the bound is exact for $d = 1$ and overestimates the true variance by at most 70\% for $d = 50$, i.e., by less than a factor of two. Finally, our experiments on active learning benchmarks (Section~\ref{sec:experiments}) suggest that the parameter sensitivity matrix $G_{\hat\theta}(x)$ is informative for active learning across multiple model classes.


\subsection{Transport Neural Tangent Kernel}

The parameter sensitivity matrix $G_\theta(x)$ extends naturally to a pairwise object by considering two distinct inputs $x$ and $x'$. Specifically, we define the \emph{transport Neural Tangent Kernel} (tNTK) as:
\begin{align}
    k_\theta(x, x') = \int \operatorname{Tr}\left[D_\theta T_{\theta}(x)(y)^\top D_\theta T_{\theta}(x')(y)\right]d\nu(y). \label{eq:kernel_def}
\end{align}
This kernel measures the alignment between the parameter sensitivities of the transport maps at $x$ and $x'$, integrated over the base measure $\nu$. The diagonal of the kernel recovers the trace of the parameter sensitivity matrix, $k_\theta(x, x) = \operatorname{Tr}[G_\theta(x)]$. Under an isotropic posterior $\Sigma = \sigma^2 I$ -- the standard choice in kernel-based active learning~\citep{Holzmuller2022-do} -- Theorem~\ref{thm:sensitivitybound} implies that \begin{equation}
V_F(x) \leq \sigma^2 k_\theta(x, x) + O(\sigma^3).
\end{equation} 
This identifies the kernel diagonal as a tractable scalar measure of predictive uncertainty. The tNTK is positive semi-definite (Appendix~\ref{app:kernel_psd}). When $\nu$ degenerates to a Dirac measure, $\mu_\theta(x)$ collapses to a single point and the tNTK reduces to the finite-width NTK~\citep{Khan2019-qd}, justifying the name.

\paragraph{Kernel Computation and Estimation} The tNTK may be computed over all parameters $\theta$ or restricted to the last layer parameters. When the transport map takes the form $T_\theta(x)(y) = W^\top h_\theta(x, y)$ with penultimate features $h_\theta(x, y) \in \mathbb{R}^m$ and a linear last layer $W \in \mathbb{R}^{m \times d}$, the tNTK has a simple expression:
\begin{align}
    k^\mathrm{ll}_\theta(x, x') = d \int h_\theta(x, y)^\top h_\theta(x', y) \, d\nu(y). \label{eq:kernel_lastlayer}
\end{align}
The proof appears in Appendix~\ref{app:last_layer_kernel}. \eqref{eq:kernel_lastlayer} involves the inner product of penultimate features at $x$ and $x'$ integrated over the base measure $\nu$ and therefore represents the distributional analog of the last-layer NTK~\citep{Holzmuller2022-do}. In practice, we approximate the integral with Monte Carlo samples from $\nu$, requiring only forward passes through the network $h_\theta$.

\subsection{Kernel Decomposition for Affine Gaussian Transport Maps}
\label{sec:affine_gaussian_transport}
We show that when $T_\theta(x)$ is an affine transport map, the tNTK admits a clean decomposition into mean- and shape-sensitivity terms.
\begin{restatable}{lemma}{kerneldecomposition}
\label{lem:kernel_decomposition}
Let us assume that $T_\theta(x)$ is an affine map: $T_\theta(x)(y) = b_\theta(x) + M_\theta(x)y$ and that the base measure $\nu$ is centered and isotropic (i.e., $\mathbb{E}_{y \sim \nu}[y] = 0$ and $\mathbb{E}_{y \sim \nu}[yy^\top] = I$). Then
\begin{align}
    k_\theta(x, x') = \mathrm{Tr}\left[D_\theta b_\theta(x)^\top D_\theta b_\theta(x')\right] + \sum_{j=1}^d \mathrm{Tr}\left[D_\theta M^j_\theta(x)^\top D_\theta M^j_{\theta}(x')\right],
    \label{eq:affine_gaussian_map}
\end{align}
where $M^j_\theta(x)$ denotes the $j$-th column of $M_\theta(x)$.
\end{restatable}
The proof appears in Appendix~\ref{app:proof_kernel_decomposition}.  
This decomposition mirrors the one from Proposition~\ref{prop:risk_decomposition} at the kernel level: the first term captures sensitivity of the predicted mean, the second captures sensitivity of the predicted shape. Additionally, we demonstrate that when $M_\theta(x)$ is further restricted to be lower-triangular (e.g., via Cholesky decomposition), for any perturbation $h \in \mathbb{R}^p$, $h^\top G_\theta(x)h$ coincides with the squared metric derivative at 0 of the curve $t \mapsto \mu_{\theta + t h}(x)$ under the Knothe--Rosenblatt (KR) transport metric (Appendix~\ref{app:KR}). This suggests that in the setting of linear maps, the parameter sensitivity matrix $G_\theta(x)$ effectively captures shape variations.

%


\subsection{Active Learning Methods via the Transport NTK}

Let $\mathcal{X}_\textrm{train}$ denote the conditions for which we have observed output samples, and $\mathcal{X}_{\textrm{pool}}$ the set of inputs for which experimental data is not available. Treating the tNTK as the covariance function of a Gaussian process, the GP posterior variance at a test point $x'$ after conditioning on a set $\mathcal{X}_S \subseteq \mathcal{X}$ is:
\begin{equation}
    \sigma^2_\theta(x' \mid \mathcal{X}_S) = k_\theta(x', x') - k_\theta(\mathcal{X}_S, x')^\top \left(k_\theta(\mathcal{X}_S, \mathcal{X}_S) + \lambda I\right)^{-1} k_\theta(\mathcal{X}_S, x'),
    \label{eq:gp_posterior_variance}
\end{equation}
where $\lambda > 0$ is a regularization parameter. We greedily select the next candidate as the one that minimizes the average posterior variance over the whole dataset:
\begin{equation}
    x^* = \argmin_{x \in \mathcal{X}_\text{pool}} \; \frac{1}{|\mathcal{X}_\text{pool}| + |\mathcal{X}_\text{train}|} \sum_{x' \in \mathcal{X}_\text{pool}\cup\mathcal{X}_\text{train}} \sigma^2_\theta\!\left(x' \mid \mathcal{X}_\text{train} \cup \{x\}\right).
    \label{eq:acquisition}
\end{equation}
We apply this rule sequentially to design an entire batch (Algorithm~\ref{alg:al_transportkernel}).

\begin{algorithm}[H]
\caption{Active learning with the transport neural tangent kernel}
\label{alg:al_transportkernel}
\begin{algorithmic}[1]
\Require Pool of conditions $\mathcal{X}$, batch size $B$, number of rounds $T$, regularization $\lambda > 0$
\State Initialize $\mathcal{X}_\text{train} \subset \mathcal{X}$ with $|\mathcal{X}_\text{train}| = B_0$; set $\mathcal{X}_\text{pool} = \mathcal{X} \setminus \mathcal{X}_\text{train}$
\For{$t = 1, \ldots, T$}
    \State Fit generative model on $\mathcal{X}_\text{train}$ to obtain parameters $\theta_t$
    \State $\mathcal{B} \gets \emptyset$ \Comment{current batch}
    \State $\mathcal{X}_\mathcal{C} \gets \mathcal{X}_\text{train}$ \Comment{conditioning set}
    \For{$b = 1, 2, \ldots, B$}
    \State $x^* \gets \displaystyle\argmin_{x \in \mathcal{X}_\text{pool}\setminus\mathcal{B}}\; \sum_{x' \in \mathcal{X}} \sigma^2_{\theta_t}\!\left(x' \mid \mathcal{X}_\mathcal{C} \cup \{x\}\right)$ \Comment{select point}
    \State $\mathcal{B} \gets \mathcal{B} \cup \{x^*\}$
    \State $\mathcal{X}_\mathcal{C} \gets \mathcal{X}_\mathcal{C} \cup \{x^*\}$ 
    \EndFor
    \State $\mathcal{X}_\text{train} \gets \mathcal{X}_\text{train} \cup \mathcal{B}$ \Comment{update training set}
    \State $\mathcal{X}_\text{pool} \gets \mathcal{X}_\text{pool} \setminus \mathcal{B}$ \Comment{update pool}
\EndFor
\State \Return Final model parameters $\theta_T$
\end{algorithmic}
\end{algorithm}

\section{Experiments}
\label{sec:experiments}
We investigate the empirical performance of our active learning strategy based on the tNTK, following Algorithm \ref{alg:al_transportkernel}, across four conditional generative modeling tasks. We consider as baseline the selection rule in \eqref{eq:acquisition} applied to the NTK computed on a neural network fitting the mean of the distribution (``NTK'') and a random sampling baseline (see Appendix \ref{app:ntk_from_pushforward} for more details). The NTK baseline matches BAIT~\citep{Ash2021-zv} as defined in \citet{Holzmuller2022-do}, and is among the best active learning methods for regression tasks in \citet{Holzmuller2022-do}. Unless mentioned otherwise, we use the Monte Carlo estimate of the last layer tNTK for all experiments with finite data (using $M=4,096$ samples). Full experimental details can be found in Appendix~\ref{app:exp}.

\begin{figure}[b]
  \centering
  \includegraphics[width=\textwidth]{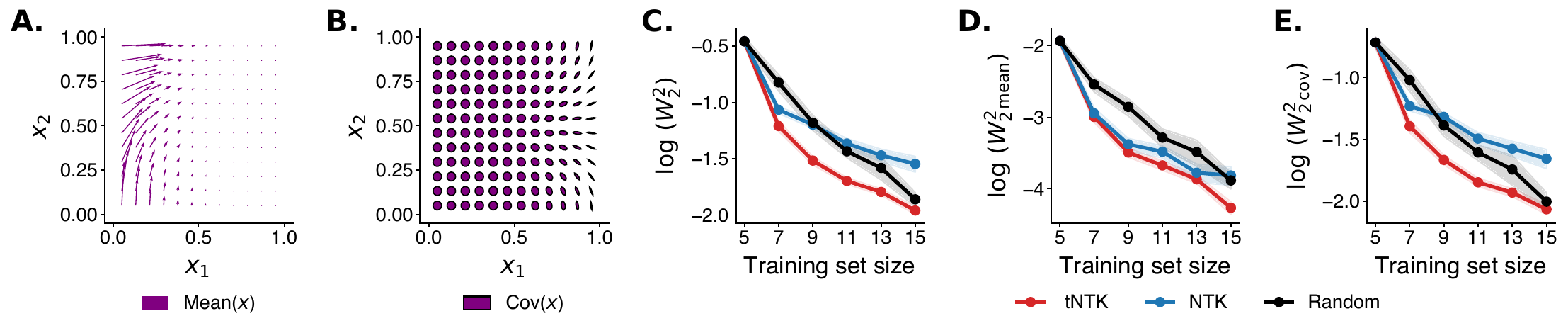}
  \caption{Results on the synthetic Gaussian 2D dataset. (A) Ground truth conditional mean $\mathbb{E}[y \mid x]$. (B) Ground truth conditional covariance $\text{Cov}[y \mid x]$. (C) Test Wasserstein error after active learning as a function of training set size (lower is better). (D-E) Mean (D) and covariance (E) components of the test Wasserstein error. Results are averaged over $30$ independent random initializations; shaded bands show standard errors. 
  }
  \label{fig:gaussian2d}
\end{figure}

\subsection{Affine transport of 2D Gaussian distributions}
\label{sec:result_affine_2d_gaussian}
We first investigate the active learning performance of the kernel when $G_\theta(x)$ can be computed in closed form (Lemma \ref{lem:kernel_decomposition}). 
We consider 2-dimensional isotropic Gaussian $y \sim \mathcal{N}(0, I_2)$ and conditional transport maps of the form
\begin{equation}
  \label{eq:affine-transport}
  T_\theta(x)(y) \;=\; b_\theta(x) \;+\; L_\theta(x)y\, ,
\end{equation}
where $b_\theta(x) \in \mathbb{R}^2$ is the predicted mean and
$L_\theta(x) \in \mathbb{R}^{2 \times 2}$ is a lower-triangular matrix with strictly positive diagonal entries, so that $\Sigma_\theta(x) = L_\theta(x)\,L_\theta(x)^\top$ is a valid covariance matrix. $b_\theta(x)$ and $L_\theta(x)$ are parameterized by neural networks, and trained against the analytical Gaussian parameters observed for each input $x$ (under an infinite-sample assumption). We generate data so that the input region $x_1 \approx 0$ has high variation in the mean but nearly constant covariance, with the opposite pattern for the region $x_1 \approx 1$ (Figure~\ref{fig:gaussian2d}A-B). While the mean-based NTK does not improve over random selection, the tNTK achieves a 1.2x lower Wasserstein loss than random by the end of the campaign (Figure~\ref{fig:gaussian2d}C). As in the one-dimensional example of Section~\ref{sec:histogram_1d}, NTK improves over random in mean loss but not in covariance loss (Figure~\ref{fig:gaussian2d}D-E). We further show that this result also holds for the affine transport of $d$-dimensional Gaussian distributions, where the tNTK outperforms NTK up to $d=50$ (Appendix~\ref{app:gaussian_nd}). Model and experimental details appear in Appendix~\ref{app:affine-kernel}.

\subsection{Conditional quantile prediction}
\label{sec:exp_quantile}
We next consider the case where $\text{dim}(\mathcal{Y}) = 1$ with $T_\theta(x)$ monotone in $y$. As shown in Appendix~\ref{app:proof_sensitivity_bound}, this yields the quantile representation $T_\theta(x) = F^{-1}_\theta(x;\cdot) \circ F_\nu$, where $F^{-1}_\theta(x;\cdot)$ is the quantile function of $\mu_\theta(x)$. In this setting, the bound from Theorem~\ref{thm:sensitivitybound} is tight at the leading order, and the tNTK admits the following closed-form expression:
\begin{equation}
    k_\theta^{\text{1D}}(x, x') = \int_0^1 \nabla_\theta F^{-1}_\theta(x;\tau)^\top \nabla_\theta F^{-1}_\theta(x';\tau)\, d\tau.
\end{equation}
We discretize $\tau$ on a uniform grid and parameterize conditional quantiles $F_\theta^{-1}(x;\tau)$ with non-crossing quantile networks~\citep{Moon2021-mu}. We evaluate the performance of our method on one simulated and two real datasets. Simulated data are based on three-component Gaussian mixture models whose parameters are functions of the input $x$. Real data are based on yearly distributions of surface temperature and dew point~\citep{noaa_gsod_1999}, predicted from the location of the measuring station. 

In this experiment, we include as baseline the NTK from multiple networks that were fitted on summary statistics of the data, such as the mean, the mean and the variance (mean-var), as well as all quartiles (quartile) of the distribution. We also report results for Bayesian Active Learning By Distribution Disagreement (BALSA\textsuperscript{EMD})~\citep{Werner2025-qo}. 

The tNTK consistently outperforms all baselines across differing numbers of initial conditions ($B_0$) and the number of conditions selected in each iteration ($B$) (Table \ref{tab:quantile_result} and Figure~\ref{fig:supp_quantile_w2}). 
Model and dataset details appear in Appendix~\ref{app:exp_quantile}.

\begin{table}[t]
\centering
\small
\caption{Change in Wasserstein test risk from iteration 0 (lower is better). For each seed, we average over $T=5$ iterations and $K=5$ folds; we then report mean $\pm$ standard error across $N=10$ seeds, in units of $10^{-3}$. Column headers indicate $B_0 + B$, the initial training set size and the number of conditions added per AL round. Best mean per column shown in bold.}
\label{tab:quantile_result} \vspace{0.2cm}
\begin{tabular}{llccc}
\toprule
Dataset & Strategy & 5 + 1 & 5 + 5 & 10 + 1 \\
\midrule
\multirow{6}{*}{Gaussian mixture} & Random & $-296.12 \pm 16.29$ & $-609.61 \pm 17.10$ & $-121.02 \pm 10.93$ \\
 & NTK & $-416.43 \pm 16.15$ & $-685.12 \pm 17.58$ & $-168.28 \pm 11.23$ \\
 & NTK (mean-var) & $-367.99 \pm 17.59$ & $-686.29 \pm 17.39$ & $-162.85 \pm 11.04$ \\
 & NTK (quartile) & $-388.93 \pm 16.43$ & $-681.29 \pm 17.62$ & $-176.86 \pm 10.72$ \\
 & BALSA-EMD & $-332.16 \pm 16.06$ & $-552.52 \pm 17.07$ & $-133.59 \pm 12.12$ \\
 & tNTK & $\mathbf{-429.52 \pm 16.39}$ & $\mathbf{-695.82 \pm 17.89}$ & $\mathbf{-205.01 \pm 12.38}$ \\
\midrule
\multirow{6}{*}{Temperature} & Random & $-89.04 \pm 12.88$ & $-219.96 \pm 14.43$ & $-41.02 \pm 7.44$ \\
 & NTK & $-185.55 \pm 14.66$ & $-255.65 \pm 15.29$ & $-84.57 \pm 9.70$ \\
 & NTK (mean-var) & $-186.51 \pm 15.31$ & $-258.48 \pm 15.55$ & $-86.88 \pm 9.69$ \\
 & NTK (quartile) & $-182.63 \pm 14.97$ & $-255.01 \pm 15.86$ & $-88.36 \pm 9.66$ \\
 & BALSA-EMD & $-155.37 \pm 13.64$ & $-199.14 \pm 15.87$ & $-71.14 \pm 9.57$ \\
 & tNTK & $\mathbf{-189.94 \pm 15.22}$ & $\mathbf{-267.93 \pm 15.36}$ & $\mathbf{-89.84 \pm 10.01}$ \\
\midrule
\multirow{6}{*}{Dew Point} & Random & $-79.90 \pm 14.55$ & $-280.77 \pm 21.75$ & $-76.55 \pm 10.15$ \\
 & NTK & $-178.44 \pm 18.98$ & $-364.19 \pm 21.53$ & $-158.94 \pm 10.97$ \\
 & NTK (mean-var) & $-172.47 \pm 17.97$ & $-372.49 \pm 22.06$ & $-159.24 \pm 10.32$ \\
 & NTK (quartile) & $-173.21 \pm 18.59$ & $-375.47 \pm 22.57$ & $-158.70 \pm 11.14$ \\
 & BALSA-EMD & $-79.51 \pm 14.23$ & $-178.00 \pm 17.91$ & $-124.13 \pm 9.39$ \\
 & tNTK & $\mathbf{-221.18 \pm 17.13}$ & $\mathbf{-442.25 \pm 22.34}$ & $\mathbf{-212.74 \pm 11.73}$ \\
\bottomrule
\end{tabular}
\end{table}

\subsection{Conditional Flow Matching}
\label{sec:exp_cflow}
In conditional flow matching, the transport map is the time-1 flow $T_\theta(x)(y) = \phi^{1, x}(y)$ obtained by integrating the ODE from Equation~\ref{eq:cfm_ode}. We approximate the total derivative by the integration of $D_\theta v_\theta^{t, x}$, and compute the tNTK as
\begin{equation}
k^{\text{Flow}}_\theta(x, x') = \int \mathrm{Tr}\!\left[\left(\int_0^1 D_\theta v_\theta^{t, x}(\phi^{t, x}(y))\, dt\right)^{\!\top} \left(\int_0^1 D_\theta v_\theta^{t, x'}(\phi^{t, x'}(y))\, dt\right)\right] d\nu(y).
\label{eq:flow_kernel}
\end{equation}
We estimate the integral over $\nu$ in \eqref{eq:flow_kernel} with Monte Carlo sampling, and compute both the inner time integration and the flow ODE for $\phi^{t,x}(y)$ with Tsitouras' 5/4 adaptive Runge--Kutta method~\citep{tsitouras2011runge}. For this experiment, we use a Gaussian conditional probability path with mean $\mu_t = t\, y_1$ and standard deviation $\sigma_t = 1 - t$, where $y_1 \sim \mu(x)$ is a target sample and the path interpolates between $\mathcal{N}(0, I)$ at $t = 0$ and a Dirac at $y_1$ at $t = 1$, following \citet{Lipman2022-mt}.

We evaluate our kernel on two simulated and one real-world dataset. The simulated systems generate particle distributions in $\mathbb{R}^2$: radial flows with attractors at varying radii and steepness, and Van der Pol oscillators~\citep{vanderpol1926relaxation} with varying parameters and initial particle distributions; we use 100 conditions for each. For the real-world data, we use the CLINC150 intent classification dataset~\citep{larson2019,clinc150_570}: each conditioning input is the fine-tuned MINILM~\citep{10.5555/3495724.3496209} text embedding of an intent class label, and the target distribution is the empirical distribution of sentences in that class, embedded into an 8-dimensional latent space using PCA on BART~\citep{Lewis2019-kc} sentence embeddings. 

tNTK outperforms NTK across 8 out of 9 (dataset, budget) configurations (Table~\ref{tab:cflow_result} and Figure~\ref{fig:supp_cflow_sinkhorn}). In the remaining configuration (CLINC 150; $B_0=5$; $B=5$), NTK does not significantly outperform tNTK ($p>0.5$, paired t-test). This result demonstrates that the kernel remains useful even when its computation requires numerical integration of the flow ODE. Full experimental details appear in Appendix~\ref{app:exp_cflow}.

\begin{table}[t]
\centering
\small
\caption{Change in Sinkhorn divergence from iteration 0 (lower is better). For each seed, we average over $T=3$ iterations and $K=5$ folds; we then report mean $\pm$ standard error across $N=10$ seeds, in units of $10^{-3}$. Column headers indicate $B_0 + B$, the initial training set size and the number of conditions added per AL round. Best mean per column shown in bold.}
\label{tab:cflow_result} \vspace{0.2cm}
\begin{tabular}{llccc}
\toprule
Dataset & Strategy & 5 + 1 & 5 + 5 & 10 + 5 \\
\midrule
\multirow{3}{*}{Radial flow} & Random & $-17.70 \pm 3.97$ & $-40.24 \pm 8.42$ & $-24.60 \pm 4.16$ \\
 & NTK & $-26.71 \pm 6.35$ & $-43.07 \pm 7.61$ & $-24.24 \pm 4.86$ \\
 & tNTK & $\mathbf{-35.50 \pm 5.36}$ & $\mathbf{-53.55 \pm 6.83}$ & $\mathbf{-38.39 \pm 5.03}$ \\
\midrule
\multirow{3}{*}{Van der Pol oscillator} & Random & $-121.87 \pm 31.66$ & $-336.33 \pm 34.17$ & $-171.39 \pm 28.31$ \\
 & NTK & $-196.36 \pm 32.59$ & $-408.15 \pm 37.03$ & $-220.98 \pm 33.08$ \\
 & tNTK & $\mathbf{-200.73 \pm 37.15}$ & $\mathbf{-441.51 \pm 24.25}$ & $\mathbf{-259.98 \pm 37.37}$ \\
\midrule
\multirow{3}{*}{CLINC 150} & Random & $-350.07 \pm 91.54$ & $-883.23 \pm 155.74$ & $-436.98 \pm 133.64$ \\
 & NTK & $-306.01 \pm 111.84$ & $\mathbf{-884.02 \pm 162.84}$ & $-480.66 \pm 96.75$ \\
 & tNTK & $\mathbf{-672.85 \pm 107.66}$ & $-852.31 \pm 133.11$ & $\mathbf{-580.94 \pm 93.76}$ \\
\bottomrule
\end{tabular}
\end{table}

\begin{wrapfigure}{r}{0.35\linewidth}
    \centering
    \vspace{-0.55in}
    \includegraphics[width=\linewidth]{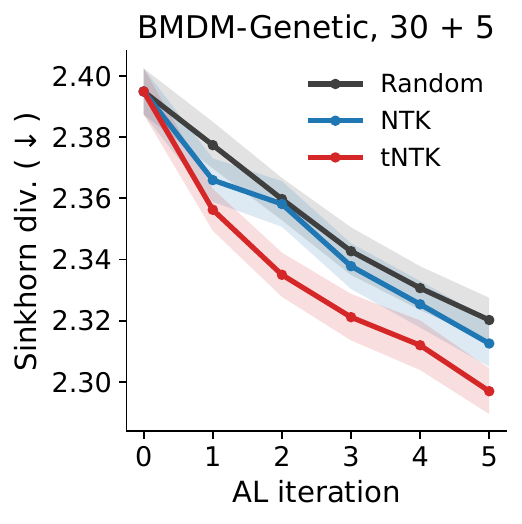}
    \caption{Mean Sinkhorn divergence per active learning iteration across 5 folds and 10 random seeds. Shade shows standard error.}
    \label{fig:cmonge_example_panel}
    \vspace{-0.5in}
\end{wrapfigure}
\subsection{Conditional Monge Gap}
\label{sec:exp_cmonge}

We use the conditional Monge gap~\citep{Driessen2025-vw,Uscidda2023-as} to predict single-cell responses in two perturbation datasets, where mapping unperturbed cell states to perturbed cell states has often been modeled using optimal transport~\citep{Schiebinger2019-ij,Bunne2023-jf}. The K562-Chemical task uses the SciPlex3 dataset~\citep{Srivatsan2020-vn}: $x$ is the mean transcriptomic response in MCF7 cells (adenocarcinoma cell line) under a small-molecule drug treatment, and the target is the single-cell transcriptomic response in K562 cells (leukemia cell line) under the same treatment, embedded into a 10-dimensional latent space. The BMDM-Genetic task uses a genetic screen of bone marrow-derived macrophages: $x$ is the imaging outcome from a matched optical pooled screen, and the target is the transcriptomic response in a 10-dimensional latent space. Data processing details appear in Appendix~\ref{app:exp_cmonge}. 

tNTK outperforms NTK and random baselines in 9 out of 12 (dataset, budget) configurations (Table~\ref{tab:cmonge_result} and Figure~\ref{fig:supp_cmonge_sinkhorn}). In the remaining cases, NTK did not outperform significantly the tNTK ($p > 0.1$, paired t-test).

\begin{table}[t]
\centering
\small
\caption{Change in Sinkhorn divergence from iteration 0 (lower is better). For each seed, we average over $T=5$ iterations and $K=5$ folds; we then report mean $\pm$ standard error across $N=10$ seeds, in units of $10^{-3}$. Column headers indicate $B_0 + B$, the initial training set size and the number of conditions added per AL round. Best mean per column shown in bold.}
\label{tab:cmonge_result} \vspace{0.2cm}
\begin{tabular}{llccc}
\toprule
Dataset & Strategy & \multicolumn{3}{c}{$B_0 + B$} \\
\midrule
\multirow{8}{*}{K562-Chemical} & & 10 + 1 & 10 + 5 & 20 + 5 \\
\cmidrule(lr){3-5}
 & Random & $\mathbf{-9.56 \pm 0.69}$ & $-29.06 \pm 1.09$ & $-15.27 \pm 0.69$ \\
 & NTK & $-9.33 \pm 0.75$ & $\mathbf{-30.44 \pm 1.45}$ & $-20.04 \pm 0.80$ \\
 & tNTK & $-8.65 \pm 1.16$ & $-30.26 \pm 1.48$ & $\mathbf{-21.00 \pm 1.30}$ \\
\cmidrule(lr){2-5}
 & & 20 + 10 & 30 + 5 & 30 + 10 \\
\cmidrule(lr){3-5}
 & Random & $-24.14 \pm 1.29$ & $-10.70 \pm 1.19$ & $-17.19 \pm 1.34$ \\
 & NTK & $-30.86 \pm 0.84$ & $-17.57 \pm 1.06$ & $-25.53 \pm 0.97$ \\
 & tNTK & $\mathbf{-31.63 \pm 1.22}$ & $\mathbf{-19.30 \pm 1.18}$ & $\mathbf{-26.13 \pm 1.13}$ \\
\midrule
\multirow{8}{*}{BMDM-Genetic} & & 10 + 1 & 10 + 5 & 20 + 5 \\
\cmidrule(lr){3-5}
 & Random & $-23.67 \pm 5.56$ & $-98.81 \pm 6.38$ & $-67.70 \pm 5.29$ \\
 & NTK & $\mathbf{-38.16 \pm 4.98}$ & $-127.30 \pm 6.24$ & $-92.80 \pm 5.00$ \\
 & tNTK & $-32.01 \pm 4.33$ & $\mathbf{-130.46 \pm 7.60}$ & $\mathbf{-102.63 \pm 5.20}$ \\
\cmidrule(lr){2-5}
 & & 20 + 10 & 30 + 5 & 30 + 10 \\
\cmidrule(lr){3-5}
 & Random & $-111.13 \pm 3.87$ & $-48.73 \pm 2.41$ & $-79.70 \pm 3.97$ \\
 & NTK & $-130.32 \pm 5.70$ & $-54.90 \pm 4.17$ & $-90.86 \pm 4.54$ \\
 & tNTK & $\mathbf{-144.90 \pm 3.50}$ & $\mathbf{-70.59 \pm 4.84}$ & $\mathbf{-106.25 \pm 4.91}$ \\
\bottomrule
\end{tabular}
\end{table}

\section{Discussion}
We introduced an active learning framework for conditional distributions based on the \textit{transport neural tangent kernel}. The tNTK captures the sensitivity of a conditional generative model's predicted distribution to perturbations of its parameters, and generalizes the NTK to distributional outputs. We derived the tNTK as an upper bound on the posterior Fréchet variance (Theorem~\ref{thm:sensitivitybound}) and provided a geometric interpretation on the Wasserstein manifold (Appendix~\ref{app:geometry}). Across four conditional generative model classes and on simulated and real datasets, total-posterior-variance-minimizing selection with the tNTK as GP covariance function outperforms the same strategy applied to the NTK of the conditional mean, as well as random sampling.

Sequential greedy total posterior variance minimization, while a principled choice, may be outperformed by selection methods explicitly designed for batch diversity~\citep{Holzmuller2022-do}, especially in large-batch regimes. However, the tNTK can be paired with any kernel-based selection method, so combining it with existing or novel selection strategies offers a natural extension of our framework to large-batch settings. 

The framework extends naturally to other model classes, including diffusion models: in their probability-flow ODE formulation~\citep{Song2021-ff}, the model is deterministic and the same Jacobian-integration argument as in CFM applies. Conditional normalizing flows~\citep{Winkler2019-yr} provide direct access to the inverse map and pushforward density, and score-based kernels with theoretical guarantees in the regression setting~\citep{Chaudhuri2015-xv,ash2020deep,Holzmuller2022-do} could be extended to the distributional case as an alternative to the tNTK.

In dimensions greater than one, neither the Fréchet variance bound (Theorem~\ref{thm:sensitivitybound}) nor the metric derivative bound (Lemma~\ref{lem:bound_gap}) is generally tight. While we have characterized the gap and tightness conditions for the metric derivative, our experiments evaluate output dimensions only up to $d=50$. How the tNTK behaves at substantially higher dimensions remains an open question.

\begin{ack}
We thank Jiequn Han for early discussions about distribution-aware active learning. This work was
supported through the NYU IT High Performance Computing (HPC) resources, services, and staff
expertise. Jayoung Ryu was partially supported by a Betty Moore Postdoctoral Fellowship, a Google TPU Research \& Education Award. Romain Lopez and Jayoung Ryu were partially supported by a Schmidt Sciences AI2050 Fellowship Award. 
Kyunghyun Cho was supported by the Institute of Information \& Communications Technology Planning \& Evaluation (IITP) with a grant funded by the Ministry of Science and ICT (MSIT) of the Republic of Korea in connection with the Global AI Frontier Lab International Collaborative Research Agreement. (No. RS-2024-00469482 \& RS-2024-00509279) Kyunghyun Cho was partly supported in part by the NYUAD Center for Interdisciplinary Data Science \& AI (CIDSAI), funded by Tamkeen under the NYUAD Research Institute Award CG016.

Disclosures: Romain Lopez has been an employee of Genentech since August 31st, 2021.


\end{ack}

\bibliographystyle{unsrtnat}
{\bibliography{biblio}}

\newpage 

\appendix

\appendix

\renewcommand{\thefigure}{S\arabic{figure}}
\renewcommand{\thetable}{S\arabic{table}}
\setcounter{figure}{0}
\setcounter{table}{0}

\subsection*{Code Availability Statement}
The code to reproduce the figures and tables, as well as to run the model can be found at
\url{https://github.com/biologicalML/transport_ntk}.

\section*{Supplementary Materials}

This appendix is organized as follows. 
\begin{itemize}
    \item Appendix~\ref{app:proofs_tntk}: Theoretical derivations for the motivation for distribution-aware active learning and the properties of the tNTK.
        \begin{itemize}
            \item Appendix~\ref{app:proof_risk_decomposition}: Proof of Proposition~\ref{prop:risk_decomposition} on risk decomposition.
            \item Appendix~\ref{app:proof_risk_1d}: Proof of Lemma~\ref{lem:quantile_variance} on mean-variance decomposition of risk in one dimension.
            \item Appendix~\ref{app:proof_sensitivity_bound}: Proof of Theorem~\ref{thm:sensitivitybound} on the upper bound of the Fréchet variance.
            \item Appendix~\ref{app:kernel_psd}: Positive semi-definiteness of the tNTK.
            \item Appendix~\ref{app:last_layer_kernel}: Simplified kernel computation for models with a linear output layer, using sensitivity to last-layer parameters.
            \item Appendix~\ref{app:proof_kernel_decomposition}: Proof of Lemma~\ref{lem:kernel_decomposition} on kernel decomposition for affine transport.
        \end{itemize}
    
    \item Appendix~\ref{app:geometry}: Geometric interpretation of the tNTK as an upper bound on the metric derivative on the 2-Wasserstein space, and relationship to metric derivative on the Knothe–Rosenblatt distance.
        \begin{itemize}
            \item Appendix~\ref{app:metric_background}: Background on curves and metric derivatives on Wasserstein space.
            \item Appendix~\ref{app:metric_bound}: The norm induced by $G_\theta(x)$ upper-bounds the metric derivative of a curve induced by parameter perturbation.
            \item Appendix~\ref{app:gap_analysis}: Analysis of the gap between the bound and the metric derivative, and conditions under which it is tight.
        \end{itemize}
    \item Appendix~\ref{app:KR}: Exact metric tensor under the Knothe--Rosenblatt distance: the tNTK recovers the squared metric derivative for lower-triangular affine maps.
    \item Appendix~\ref{app:exp}: Experimental details.
        \begin{itemize}
            \item Appendix~\ref{app:condhistdetails}: Conditional histogram example from Section~\ref{sec:histogram_1d}.
            \item Appendix~\ref{app:ntk_from_pushforward}: NTK baseline calculation.
            \item Appendix~\ref{app:exp_2d_gaussian}: Affine transport of 2D Gaussians (Section~\ref{sec:result_affine_2d_gaussian}).
            \item Appendix~\ref{app:exp_quantile}: Conditional quantile prediction (Section~\ref{sec:exp_quantile}).
            \item Appendix~\ref{app:exp_cflow}: Conditional flow matching (Section~\ref{sec:exp_cflow}).
            \item Appendix~\ref{app:exp_cmonge}: Conditional Monge gap (Section~\ref{sec:exp_cmonge}).
        \end{itemize}
    \item{Appendix~\ref{app:additional_results}: Additional results.}
        \begin{itemize}
            \item Appendix~\ref{app:numerical_gap_analysis}: Empirical calculation of gap of the upper bound in Theorem~\ref{thm:sensitivitybound}.
            \item Appendix~\ref{app:gaussian_nd}: Active learning results on the higher-dimensional affine Gaussian transport.
            \item Appendix~\ref{app:tntk_computation_results}: Results on tNTK computation.
            \item Appendix~\ref{app:additional_baselines}: Additional baselines for the experiments presented in the main text.
        \end{itemize}
    \item{Appendix~\ref{app:supplementary_figs}: Supplementary Figures.}
\end{itemize}

\clearpage
\section{Theory of the Transport Neural Tangent Kernel}
\label{app:proofs_tntk}
\subsection{Proof of Proposition~\ref{prop:risk_decomposition}}
\label{app:proof_risk_decomposition}
\riskdecomposition*
\begin{proof}
We drop the dependence on $x$ for readability. For any vector $v \in \mathcal{Y}$, we denote by $\tau_v : \mathcal{Y} \to \mathcal{Y}$ the translation operator $\tau_v(y) = y - v$. For any measure $\rho \in \mathcal{P}_2(\mathcal{Y})$ with mean $m_\rho$, the centered measure is then $\rho^c = (\tau_{m_\rho})\# \rho$. We extend this notation to product spaces by writing $\tau_{v, v'}(y, y') = (y - v, y' - v')$ for the translation acting on couplings.

Let $\pi \in \Pi(\mu, \mu_\theta)$ be an optimal coupling for $W_2^2(\mu, \mu_\theta)$, and let $\pi^c \in \Pi(\mu^c, \mu^c_\theta)$ be an optimal coupling for $W_2^2(\mu^c, \mu^c_\theta)$. We will use the following algebraic identity twice:
\begin{equation}
    \| (y - m) - (y' - m_\theta)\|^2 = \|y - y' \|^2 + \|m_\theta - m\|^2 + 2(m_\theta - m)^\top(y - y'). \label{eq:identity}
\end{equation}

\paragraph{Step 1: $(\tau_{m, m_\theta})_\# \pi$ is admissible for $\mu^c, \mu^c_\theta$.} Define $\hat{\pi} = (\tau_{m, m_\theta})_\# \pi$. Since $\pi$ has marginals $\mu$ and $\mu_\theta$, the pushforward $\hat{\pi}$ has marginals $(\tau_m)\# \mu = \mu^c$ and $(\tau_{m_\theta})_\# \mu_\theta = \mu^c_\theta$, so $\hat{\pi} \in \Pi(\mu^c, \mu^c_\theta)$.

Applying~\eqref{eq:identity} with $u = y - m$ and $u' = y' - m_\theta$, and integrating against $\pi$, using $\mathbb{E}_\pi[y] = m$ and $\mathbb{E}_\pi[y'] = m_\theta$:
\begin{align}
    \iint \|u - u'\|^2 \, d\hat{\pi}(u, u') 
    &= \iint \|(y - m) - (y' - m_\theta)\|^2 \, d\pi(y, y') \\
    &= \underbrace{\iint \|y - y'\|^2 \, d\pi}_{= \, W_2^2(\mu, \mu_\theta)} + \|m_\theta - m\|^2 + 2(m_\theta - m)^\top \underbrace{\mathbb{E}_\pi[y - y']}_{= \, m - m_\theta} \\
    &= W_2^2(\mu, \mu_\theta) - \|m_\theta - m\|^2.
\end{align}
Since $\hat{\pi}$ is admissible for $\mu^c, \mu^c_\theta$ but not necessarily optimal,
\begin{equation}
    W_2^2(\mu^c, \mu^c_\theta) \leq W_2^2(\mu, \mu_\theta) - \|m_\theta - m\|^2. \label{eq:ineq1}
\end{equation}

\paragraph{Step 2: $(\tau_{-m, -m_\theta})_\# \pi^c$ is admissible for $\mu, \mu_\theta$.} Symmetrically, define $\tilde{\pi} = (\tau_{-m, -m_\theta})\# \pi^c$. Since $\pi^c$ has marginals $\mu^c$ and $\mu^c_\theta$, the pushforward $\tilde{\pi}$ has marginals $\mu$ and $\mu_\theta$, so $\tilde{\pi} \in \Pi(\mu, \mu_\theta)$.

Applying~\eqref{eq:identity} with $y = u + m$ and $y' = u' + m_\theta$, and integrating against $\pi^c$, using $\mathbb{E}_{\pi^c}[u] = 0$ and $\mathbb{E}_{\pi^c}[u'] = 0$:
\begin{align}
    \iint \|y - y'\|^2 \, d\tilde{\pi}(y, y') 
    &= \iint \|(u + m) - (u' + m_\theta)\|^2 \, d\pi^c(u, u') \\
    &= \underbrace{\iint \|u - u'\|^2 \, d\pi^c}_{= \, W_2^2(\mu^c, \mu^c_\theta)} + \|m_\theta - m\|^2 + 2(m_\theta - m)^\top \underbrace{\mathbb{E}_{\pi^c}[u' - u]}_{= \, 0} \\
    &= W_2^2(\mu^c, \mu^c_\theta) + \|m_\theta - m\|^2.
\end{align}
Since $\tilde{\pi}$ is admissible for $\mu, \mu_\theta$ but not necessarily optimal,
\begin{equation}
    W_2^2(\mu, \mu_\theta) \leq W_2^2(\mu^c, \mu^c_\theta) + \|m_\theta - m\|^2. \label{eq:ineq2}
\end{equation}

\paragraph{Step 3: Conclusion.} Combining~\eqref{eq:ineq1} and~\eqref{eq:ineq2},
\begin{equation}
    W_2^2(\mu, \mu_\theta) = W_2^2(\mu^c, \mu^c_\theta) + \|m_\theta - m\|^2.
\end{equation}
Reinstating the dependence on $x$ and taking expectations over $\theta \mid \mathcal{D}$:
\begin{align}
    \mathcal{R}(x, \mathcal{D}) 
    &= \mathbb{E}_{\theta \mid \mathcal{D}}\!\left[W_2^2(\mu_\theta(x), \mu(x))\right] \\
    &= \mathbb{E}_{\theta \mid \mathcal{D}}\!\left[\|m_\theta(x) - m(x)\|^2\right] + \mathbb{E}_{\theta \mid \mathcal{D}}\!\left[W_2^2(\mu^c_\theta(x), \mu^c(x))\right] \\
    &= \mathcal{R}_\text{mean}(x, \mathcal{D}) + \mathcal{R}_\text{shape}(x, \mathcal{D}),
\end{align}
which proves the proposition.
\end{proof}

\subsection{Proof of Lemma \ref{lem:quantile_variance}}
\label{app:proof_risk_1d}
\quantilevariance*
\begin{proof}
We drop the dependence on $x$ for readability. Throughout, all variances, covariances, and expectations are taken with respect to the posterior $\theta \mid \mathcal{D}$.

\paragraph{Step 1: Quantile representation of $W_2^2$.} For one-dimensional distributions, the squared 2-Wasserstein distance admits the closed-form quantile representation:
\begin{equation}
    W_2^2(\mu_\theta, \mu) = \int_0^1 \left( F_\theta^{-1}(\tau) - F^{-1}(\tau) \right)^2 d\tau, \label{eq:quantile_w2}
\end{equation}
where we remind the reader that $F^{-1}(\tau)$ (resp. $F_\theta^{-1}(\tau)$) corresponds to the $\tau$ quantile of the distribution $\mu$ (resp. $\mu_\theta$).

\paragraph{Step 2: Derivation of the mean risk.} Under the assumption $\mathbb{E}[m_\theta] = m$, the squared error reduces to the posterior variance:
\begin{equation}
    \mathcal{R}_\text{mean}(x, \mathcal{D}) = \mathbb{E}\!\left[(m_\theta - m)^2\right] = \operatorname{Var}[m_\theta]. \label{eq:mean_risk_1d}
\end{equation}

\paragraph{Step 3: Derivation of the shape risk.} The centered measures $\mu^c_\theta$ and $\mu^c$ have quantile functions $g_\theta(\tau) = F_\theta^{-1}(\tau) - m_\theta$ and $g(\tau) = F^{-1}(\tau) - m$, respectively. Applying~\eqref{eq:quantile_w2} to the centered measures:
\begin{equation}
    W_2^2(\mu^c_\theta, \mu^c) = \int_0^1 (g_\theta(\tau) - g(\tau))^2 \, d\tau. \label{eq:1d_shape_loss}
\end{equation}
Taking expectations and exchanging the integral and expectation by Fubini's theorem:
\begin{equation}
    \mathcal{R}_\text{shape}(x, \mathcal{D}) = \int_0^1 \mathbb{E}\!\left[(g_\theta(\tau) - g(\tau))^2\right] d\tau.
\end{equation}

Because we have that $\mathbb{E}[m_\theta] = m$ and that for all $\tau \in [0, 1]$, $\mathbb{E}[F_\theta^{-1}(\tau)] = F^{-1}(\tau)$, it implies that $\mathbb{E}[g_\theta(\tau)] = g(\tau)$ by   linearity of the expectation. Therefore, the squared error reduces to the posterior variance:
\begin{equation}
    \mathbb{E}\!\left[(g_\theta(\tau) - g(\tau))^2\right] = \operatorname{Var}[g_\theta(\tau)],
\end{equation}
and consequently
\begin{equation}
    \mathcal{R}_\text{shape}(x, \mathcal{D}) = \int_0^1 \operatorname{Var}[g_\theta(\tau)]\, d\tau. \label{eq:shape_risk_1d}
\end{equation}

Combining~\eqref{eq:mean_risk_1d} and~\eqref{eq:shape_risk_1d} proves the lemma.
\end{proof}

\subsection{Proof of Theorem \ref{thm:sensitivitybound}}
\label{app:proof_sensitivity_bound}
\sensitivitybound*
\begin{proof} 
The proof has three parts: (i) we reduce $V_F(x)$ to an integrated $L^2(\nu)$ distance between transport maps, (ii) we expand the integrand in $\Sigma$ and bound each term and collect the bounds, (iii) we analyze tightness.

\subsubsection*{Part 1: Reduction to an $L^2(\nu)$ functional}

\paragraph{Reducing $V_F$ to a single-draw expectation.}
By definition of the infimum, taking $\rho = \mu_{\hat\theta}(x)$ in \eqref{eq:frechet_variance} yields
\begin{equation}
    V_F(x) \leq \mathbb{E}_{\theta \mid \mathcal{D}}\!\left[W_2^2(\mu_\theta(x), \mu_{\hat\theta}(x))\right]. \label{eq:VF_barycenter_bound}
\end{equation}
The remainder of the proof bounds the right-hand side.

\paragraph{Synchronized coupling.}
We momentarily drop the dependence on $x$ and derive an upper bound on $W_2^2(\mu_\theta, \mu_{\hat\theta})$ for fixed $\theta$. Let $\Phi_{\theta, \hat\theta}: \mathcal{Y} \to \mathcal{Y} \times \mathcal{Y}$ denote the map $y \mapsto (T_\theta(y), T_{\hat\theta}(y))$, and introduce the synchronized coupling $\pi = (\Phi_{\theta, \hat\theta})_\# \nu$. The measure $\pi$ is a coupling of $\mu_\theta$ and $\mu_{\hat\theta}$: for any test function $\varphi$, the first marginal satisfies
\begin{align}
    \int \varphi(u)\, d\pi_1(u) = \int \varphi(T_\theta(y))\, d\nu(y) = \int \varphi(u)\, d\mu_\theta(u),
\end{align}
where we applied the pushforward identity twice, using $\mu_\theta = T_\theta\# \nu$. The same argument yields $\pi_2 = \mu_{\hat\theta}$, so $\pi \in \Pi(\mu_\theta, \mu_{\hat\theta})$. Applying the pushforward identity to the cost function $(u, u') \mapsto \|u - u'\|^2$ gives the transport cost of $\pi$:
\begin{align}
    \int \|u - u'\|^2\, d\pi(u, u') = \int \|T_\theta(y) - T_{\hat\theta}(y)\|^2\, d\nu(y).
\end{align}
Because $W_2^2$ is the infimum of the transport cost over all admissible couplings,
\begin{equation}
    W_2^2(\mu_\theta, \mu_{\hat\theta}) \leq \int \|T_\theta(y) - T_{\hat\theta}(y)\|^2\, d\nu(y). \label{eq:synchronized_bound}
\end{equation}
Reinstating the dependence on $x$ and combining with \eqref{eq:VF_barycenter_bound},
\begin{equation}
    V_F(x) \leq \mathbb{E}_{\theta \mid \mathcal{D}}\!\left[\int \|T_\theta(x)(y) - T_{\hat\theta}(x)(y)\|^2\, d\nu(y)\right]. \label{eq:V_upper}
\end{equation}

\paragraph{Taylor expansion around $\hat\theta$.}
We continue to drop the dependence on $x$ and write $J(y) := D_\theta T_{\hat\theta}(y)$ for the Jacobian at the posterior mean. By the smoothness assumption, for $\nu$-a.e.\ $y$ and all $\theta \in \mathbb{R}^p$, a second-order Taylor expansion gives
\begin{equation}
    T_\theta(y) = T_{\hat\theta}(y) + J(y)(\theta - \hat\theta) + r(\theta, y),
\end{equation}
where the remainder satisfies
\begin{equation}
    \|r(\theta, y)\| \leq \tfrac{1}{2}\, \kappa(y)\, \|\theta - \hat\theta\|^2, \qquad \kappa(y) := \sup_{\theta} \|D_\theta^2 T_\theta(y)\|_{\mathrm{op}}, \label{eq:remainder_bound}
\end{equation}
and $\kappa \in L^2(\nu)$ by assumption. Squaring and expanding the inner product,
\begin{equation}
    \|T_\theta(y) - T_{\hat\theta}(y)\|^2 = \underbrace{\|J(y)(\theta - \hat\theta)\|^2}_{=: A(y, \theta)} + \underbrace{2\, \langle J(y)(\theta - \hat\theta),\, r(\theta, y)\rangle}_{=: B(y, \theta)} + \underbrace{\|r(\theta, y)\|^2}_{=: C(y, \theta)}. \label{eq:ABC}
\end{equation}

\subsubsection*{Part 2: Term-by-term bounds}

\paragraph{Leading term.}
Since $\theta \sim \mathcal{N}(\hat\theta, \Sigma)$, we have $\mathbb{E}[(\theta - \hat\theta)(\theta - \hat\theta)^\top] = \Sigma$. Using the cyclicity of the trace,
\begin{equation}
    \mathbb{E}[A(y, \theta)] = \mathbb{E}\!\left[(\theta - \hat\theta)^\top J(y)^\top J(y)\, (\theta - \hat\theta)\right] = \operatorname{Tr}\!\left(J(y)^\top J(y)\, \Sigma\right). \label{eq:A_pointwise}
\end{equation}
By linearity of the trace operator,
\begin{equation}
    \int \mathbb{E}[A(y, \theta)]\, d\nu(y) = \operatorname{Tr}\!\left(\left(\int J(y)^\top J(y)\, d\nu(y)\right)\Sigma\right) = \operatorname{Tr}\!\left(G_{\hat\theta}(x)\, \Sigma\right). \label{eq:A_final}
\end{equation}

\paragraph{Cross term.}
By Cauchy–Schwarz applied pointwise in $y$ and then in $L^2(\mathbb{P})$,
\begin{equation}
    \mathbb{E}|B(y, \theta)| \leq 2\, \mathbb{E}\!\left[\|J(y)(\theta - \hat\theta)\|^2\right]^{1/2}\, \mathbb{E}\!\left[\|r(\theta, y)\|^2\right]^{1/2}. \label{eq:B_CS}
\end{equation}
The first factor equals $\sqrt{\operatorname{Tr}(J(y)^\top J(y)\, \Sigma)}$ by \eqref{eq:A_pointwise}. For the second factor, applying \eqref{eq:remainder_bound},
\begin{equation}
    \mathbb{E}\!\left[\|r(\theta, y)\|^2\right] \leq \tfrac{1}{4}\, \kappa(y)^2\, \mathbb{E}\!\left[\|\theta - \hat\theta\|^4\right]. \label{eq:r_squared}
\end{equation}
For $\theta \sim \mathcal{N}(\hat\theta, \Sigma)$, the standard fourth-moment identity gives $\mathbb{E}\|\theta - \hat\theta\|^4 = (\operatorname{Tr}\Sigma)^2 + 2 \operatorname{Tr}(\Sigma^2) \leq 3(\operatorname{Tr}\Sigma)^2 \leq 3 p^2 \|\Sigma\|_{\mathrm{op}}^2$, so
\begin{equation}
    \mathbb{E}\!\left[\|r(\theta, y)\|^2\right] \leq \tfrac{3}{4}\, p^2\, \kappa(y)^2\, \|\Sigma\|_{\mathrm{op}}^2. \label{eq:r_squared_bound}
\end{equation}
Substituting into \eqref{eq:B_CS} and integrating over $\nu$ via Cauchy--Schwarz,
\begin{align}
    \int \mathbb{E}|B(y, \theta)|\, d\nu(y) &\leq \sqrt{3}\, p\, \|\Sigma\|_{\mathrm{op}} \int \kappa(y)\, \sqrt{\operatorname{Tr}(J(y)^\top J(y)\, \Sigma)}\, d\nu(y) \\
    &\leq \sqrt{3}\, p\, \|\Sigma\|_{\mathrm{op}}\, \|\kappa\|_{L^2(\nu)}\, \sqrt{\operatorname{Tr}(G_{\hat\theta}(x)\, \Sigma)} \\
    &\leq \sqrt{3}\, p\, \|\kappa\|_{L^2(\nu)}\, \sqrt{\operatorname{Tr}(G_{\hat\theta}(x))}\, \|\Sigma\|_{\mathrm{op}}^{3/2}, \label{eq:B_final}
\end{align}
where the last step uses $\operatorname{Tr}(G_{\hat\theta}(x)\, \Sigma) \leq \operatorname{Tr}(G_{\hat\theta}(x))\, \|\Sigma\|_{\mathrm{op}}$ (since $G_{\hat\theta}(x)$ is positive semi-definite) and $\operatorname{Tr}(G_{\hat\theta}(x)) < \infty$ by the integrability assumption.

\paragraph{Quadratic remainder.}
Using \eqref{eq:r_squared_bound} and integrating over $\nu$,
\begin{equation}
    \int \mathbb{E}[C(y, \theta)]\, d\nu(y) \leq \tfrac{3}{4}\, p^2\, \|\kappa\|_{L^2(\nu)}^2\, \|\Sigma\|_{\mathrm{op}}^2 \leq \tfrac{3}{4}\, p^2\, \|\kappa\|_{L^2(\nu)}^2\, \|\Sigma\|_{\mathrm{op}}^{3/2}, \label{eq:C_final}
\end{equation}
where the last step uses $\|\Sigma\|_{\mathrm{op}} < 1$.

\paragraph{Combining the bounds.}
Combining \eqref{eq:V_upper}, \eqref{eq:ABC}, \eqref{eq:A_final}, \eqref{eq:B_final}, and \eqref{eq:C_final},
\begin{align}
    V_F(x) &\leq \int \mathbb{E}\!\left[A(y, \theta) + B(y, \theta) + C(y, \theta)\right] d\nu(y) \\
    &\leq \operatorname{Tr}\left(G_{\hat\theta}(x)\, \Sigma\right) + \sqrt{3}\, p\, \|\kappa\|_{L^2(\nu)}\, \sqrt{\operatorname{Tr}(G_{\hat\theta}(x))}\, \|\Sigma\|_{\mathrm{op}}^{3/2} + \tfrac{3}{4}\, p^2\, \|\kappa\|_{L^2(\nu)}^2\, \|\Sigma\|_{\mathrm{op}}^{3/2} \\
    &= \operatorname{Tr}\left(G_{\hat\theta}(x)\, \Sigma\right) + C(x)\, \|\Sigma\|_{\mathrm{op}}^{3/2},
\end{align}
where
\begin{equation}
    C(x) = \sqrt{3}\, p\, \|\kappa\|_{L^2(\nu)}\, \sqrt{\operatorname{Tr}(G_{\hat\theta}(x))} + \tfrac{3}{4}\, p^2\, \|\kappa\|_{L^2(\nu)}^2,
\end{equation}
which is the claimed bound.

\subsubsection*{Part 3: The 1D monotone case}

We show that when $\dim(\mathcal{Y}) = 1$ and $T_\theta(x)$ is monotone in $y$, the inequality \eqref{eq:sensitivity_bound} becomes tight in the leading order. In this setting, $V_F(x)$ admits a closed-form expression that we compute directly.

We drop the dependence on $x$ throughout, and write $Q_\theta(\tau) := F^{-1}_{\mu_\theta}(\tau)$ for the quantile function of $\mu_\theta$. Monotonicity of $T_\theta$ gives the quantile representation $T_\theta = Q_\theta \circ F_\nu$, and the change of variables $\tau = F_\nu(y)$ yields the isometry
\begin{equation}
    W_2^2(\mu_\theta, \mu_{\theta'}) = \int_0^1 \big(Q_\theta(\tau) - Q_{\theta'}(\tau)\big)^2\, d\tau = \|Q_\theta - Q_{\theta'}\|_{L^2[0,1]}^2. \label{eq:1d_isometry}
\end{equation}
The Wasserstein-2 geometry on $\{\mu_\theta\}$ is therefore the Euclidean geometry of $L^2[0,1]$ on $\{Q_\theta\}$, and standard $L^2$ identities give a closed form for the Fréchet barycenter and variance. The barycenter quantile function is $\bar{Q}(\tau) := \mathbb{E}_{\theta \mid \mathcal{D}}[Q_\theta(\tau)]$, and the Fréchet variance is
\begin{equation}
    V_F(x) = \mathbb{E}_{\theta \mid \mathcal{D}}\!\left[\|Q_\theta - \bar{Q}\|_{L^2[0,1]}^2\right] = \int_0^1 \operatorname{Var}_{\theta \mid \mathcal{D}}[Q_\theta(\tau)]\, d\tau, \label{eq:1d_VF}
\end{equation}
by Fubini and the $L^2$ bias-variance decomposition.

The argument of Parts 1–2 now applies pointwise in $\tau$ to $Q_\theta(\tau)$ instead of to $T_\theta(y)$, with the synchronized-coupling step replaced by the isometry \eqref{eq:1d_isometry} (which is exact rather than an inequality). Because Equation~\ref{eq:A_final} is an equality, the upper bound from Theorem~\ref{thm:sensitivitybound} is tight at leading order. More precisely, when $\|\Sigma\|_{\mathrm{op}} \rightarrow 0$:
\begin{equation}
    V_F(x) = \operatorname{Tr}\left(G_{\hat\theta}(x)\Sigma\right) + O(\|\Sigma\|_{\mathrm{op}}^{3/2}).
\end{equation}
\end{proof}

\subsection{Positive semi-definiteness of the tNTK}
\label{app:kernel_psd}

\begin{prop}
The tNTK $k_\theta : \mathcal{X} \times \mathcal{X} \to \mathbb{R}$ defined in Equation~\ref{eq:kernel_def} is positive semi-definite.
\label{prop:tntk_psd}
\end{prop}

\begin{proof}
Let $n \in \mathbb{N}^*$, and $x_1, \dots, x_n \in \mathcal{X}$ denote $n$ input points. Let $c_1, \dots, c_n \in \mathbb{R}$ be weights. Writing $J_x(y) = D_\theta T_\theta(x)(y) \in \mathbb{R}^{d \times p}$ for the parameter Jacobian, we have that:
\begin{align}
    \sum_{i=1}^n \sum_{j=1}^n c_i c_j \, k_\theta(x_i, x_j) 
    &= \sum_{i,j} c_i c_j \int \operatorname{Tr}\!\left[J_{x_i}(y)^\top J_{x_j}(y)\right] d\nu(y) \\
    &= \int \sum_{i,j} c_i c_j \left\langle J_{x_i}(y),\, J_{x_j}(y) \right\rangle_F d\nu(y) \\
    &= \int \left\| \sum_{i=1}^n c_i \, J_{x_i}(y) \right\|_F^2 d\nu(y) \geq 0,
\end{align}
where $\langle \cdot, \cdot \rangle_F$ and $\| \cdot \|_F$ denote the Frobenius inner product and norm. The final inequality holds because the squared Frobenius norm is non-negative and $\nu$ is a positive measure.
\end{proof}

\subsection{Closed-form expression for the last-layer tNTK}
\label{app:last_layer_kernel}
\begin{prop}
\label{prop:last_layer_kernel}
Let the transport map have a linear last layer, $T_\theta(x)(y) = W^\top h_\theta(x, y)$, where $h_\theta(x, y) \in \mathbb{R}^m$ is the penultimate-layer output and $W \in \mathbb{R}^{m \times d}$ is the last-layer weight matrix. Then the tNTK~\eqref{eq:kernel_def} restricted to the last-layer parameters admits the closed form:
\begin{equation}
    k^{\mathrm{ll}}_\theta(x, x') = d \int h_\theta(x, y)^\top h_\theta(x', y) \, d\nu(y).
\end{equation}
\end{prop}
\begin{proof}
For $(i, j, k) \in [m] \times [d] \times [d]$, the last-layer Jacobian satisfies
\begin{equation}
    \frac{\partial T^k_\theta(x)(y)}{\partial W_{ij}} = \delta_{kj} \, h^i_\theta(x, y),
\end{equation}
where $\delta_{kj}$ is the Kronecker delta and superscripts denote scalar components: $T^k_\theta(x)(y)$ is the $k$-th component of $T_\theta(x)(y)$, and $h^i_\theta(x, y)$ is the $i$-th component of $h_\theta(x, y)$. Substituting into \eqref{eq:kernel_def}:
\begin{align}
    k_\theta^{\mathrm{ll}}(x, x') &= \int \sum_{i,j,k} \delta_{kj} h_\theta^i(x, y) \, \delta_{kj} h_\theta^i(x', y) \, d\nu(y) \\
    &= \int \sum_{i,j} h_\theta^i(x, y) \, h_\theta^i(x', y) \, d\nu(y) \\
    &= d \int h_\theta(x, y)^\top h_\theta(x', y) \, d\nu(y),
\end{align}
where the second equality uses $\sum_k \delta_{kj}^2 = 1$, and the third uses that the summand is independent of $j$, so the sum across $j$ contributes a factor of $d$.
\end{proof}

\subsection{Proof of Lemma \ref{lem:kernel_decomposition}}
\label{app:proof_kernel_decomposition}
\kerneldecomposition*
\begin{proof}
For an affine map $T_\theta(x)(y) = b_\theta(x) + M_\theta(x) y$, the parameter Jacobian is linear in $y$:
\begin{equation}
    D_\theta T_\theta(x)(y) = D_\theta b_\theta(x) + \sum_{j=1}^d y_j \, D_\theta M^j_\theta(x), \label{eq:linear_jacobian}
\end{equation}
where $M^j_\theta(x) \in \mathbb{R}^d$ is the $j$-th column of $M_\theta(x)$ and $y_j$ is the $j$-th component of $y$. Substituting \eqref{eq:linear_jacobian} into the kernel definition \eqref{eq:kernel_def} and expanding the inner product, we get four groups of terms:
\begin{equation}
    k_\theta(x, x') = \int \operatorname{Tr}\!\left[A(x, x') + B(x, x'; y) + B(x', x; y) + C(x, x'; y)\right] d\nu(y),
\end{equation}
where
\begin{align*}
    A(x, x') &= D_\theta b_\theta(x)^\top D_\theta b_\theta(x'), \\
    B(x, x'; y) &= D_\theta b_\theta(x)^\top \sum_{j=1}^d y_j \, D_\theta M^j_\theta(x'), \\
    C(x, x'; y) &= \left(\sum_{j=1}^d y_j \, D_\theta M^j_\theta(x)\right)^\top \left(\sum_{j'=1}^d y_{j'} \, D_\theta M^{j'}_\theta(x')\right).
\end{align*}
We compute each term in turn.

\paragraph{Constant term $A$.}
$A(x, x')$ is independent of $y$, so the integral leaves it unchanged:
\begin{equation}
    \int \operatorname{Tr}[A(x, x')] \, d\nu(y) = \operatorname{Tr}\!\left[D_\theta b_\theta(x)^\top D_\theta b_\theta(x')\right].
\end{equation}

\paragraph{Cross terms $B$.}
The cross terms are linear in $y$, so under the centering assumption $\mathbb{E}[y] = 0$,
\begin{equation}
    \int \operatorname{Tr}[B(x, x'; y)] \, d\nu(y) = \operatorname{Tr}\!\left[D_\theta b_\theta(x)^\top \sum_{j=1}^d \underbrace{\mathbb{E}[y_j]}_{=\, 0} D_\theta M^j_\theta(x')\right] = 0,
\end{equation}
and similarly $\int \operatorname{Tr}[B(x', x; y)] \, d\nu(y) = 0$.

\paragraph{Quadratic term $C$.}
Expanding the product and using the trace's linearity,
\begin{equation}
    \operatorname{Tr}[C(x, x'; y)] = \sum_{j, j'=1}^d y_j y_{j'} \operatorname{Tr}\!\left[D_\theta M^j_\theta(x)^\top D_\theta M^{j'}_\theta(x')\right].
\end{equation}
Integrating against $\nu$ and using the isotropy assumption $\mathbb{E}[y y^\top] = I$, i.e., $\mathbb{E}[y_j y_{j'}] = \delta_{j j'}$:
\begin{equation}
    \int \operatorname{Tr}[C(x, x'; y)] \, d\nu(y) = \sum_{j, j'=1}^d \delta_{jj'} \operatorname{Tr}\!\left[D_\theta M^j_\theta(x)^\top D_\theta M^{j'}_\theta(x')\right] = \sum_{j=1}^d \operatorname{Tr}\!\left[D_\theta M^j_\theta(x)^\top D_\theta M^j_\theta(x')\right].
\end{equation}

\paragraph{Combining the terms.}
Summing the contributions of $A$, $B$, $B^\top$,
\begin{equation}
    k_\theta(x, x') = \mathrm{Tr}[D_\theta b_\theta(x)^\top D_\theta b_\theta(x')] + \sum_{j=1}^d \mathrm{Tr}[D_\theta M^j_\theta(x)^\top D_\theta M^j_\theta(x')]. \qedhere
\end{equation}
\end{proof}

\clearpage

\section{Geometry of parameter perturbations on the Wasserstein manifold}
\label{app:geometry}

The parameter sensitivity matrix $G_{\hat\theta}(x)$ that appears in Theorem~\ref{thm:sensitivitybound} admits a geometric interpretation as an (approximate) metric tensor on the Wasserstein manifold. This appendix makes this interpretation precise. We view a parameter perturbation $h \in \mathbb{R}^p$ as inducing a curve $t \mapsto \mu_{\theta + th}(x)$ on the 2-Wasserstein space $W_2(\mathcal{Y})$ and study its metric derivative at $t=0$. Section~\ref{app:metric_background} reviews the necessary background on metric derivatives, absolutely continuous curves, and the tangent bundle of $W_2$. Section~\ref{app:metric_bound} bounds the squared metric derivative of this curve by $h^\top G_\theta(x) h$, providing an alternative derivation of $G_\theta(x)$ via a different proof technique than Theorem~\ref{thm:sensitivitybound}. Section~\ref{app:gap_analysis} characterizes when this bound is tight: in dimension one and for monotone maps it is always exact, and in higher dimensions the gap is governed by the divergence-free component of the velocity field. The latter result points to a concrete avenue for constructing a tighter proxy in future work. 

\subsection{Background on metric derivatives in Wasserstein space}
\label{app:metric_background}

We restate the key facts about absolutely continuous curves in Wasserstein space and their characterization via the continuity equation. We work throughout in $W_2(\mathcal{Y})$ with $\mathcal{Y} = \mathbb{R}^d$. We need three preliminary objects: absolutely continuous curves, the metric derivative, and the tangent space of $\mathcal{P}_2(\mathcal{Y})$ at a measure $\mu$.

\begin{definition}[Absolutely continuous curve]
\label{def:ac_curve}
A curve $(\mu_t)_{t \in (a, b)} : (a, b) \to \mathcal{P}_2(\mathcal{Y})$ is absolutely continuous of second order on the open interval $(a, b)$ if there exists $m \in L^2(a, b)$ such that
\begin{equation}
    W_2(\mu_s, \mu_t) \leq \int_s^t m(r)\, dr \qquad \forall\, a < s \leq t < b. \label{eq:w2_bound_curve}
\end{equation}
The set of such curves is denoted by $AC^2(a, b; \mathcal{P}_2(\mathcal{Y}))$.
\end{definition}

\begin{restatable}[Metric derivative]{definition}{metricderivative}
\label{def:metric_derivative}
Let $(\mu_t)_{t \in (a, b)} \in AC^2(a, b; \mathcal{P}_2(\mathcal{Y}))$. The \emph{metric derivative} at $t \in (a, b)$ is
\begin{equation}
    |\mu'|(t) := \lim_{s \to t} \frac{W_2(\mu_s, \mu_t)}{|s - t|},
\end{equation}
whenever this limit exists, and the \emph{upper metric derivative} is
\begin{equation}
    |\mu'|^+(t) := \limsup_{s \to t} \frac{W_2(\mu_s, \mu_t)}{|s - t|},
\end{equation}
which is always defined and coincides with $|\mu'|(t)$ whenever the latter exists.
\end{restatable}

\begin{restatable}[Tangent space]{definition}{tangent}
\label{def:tangent}
Let $\mu \in \mathcal{P}_2(\mathcal{Y})$. The \emph{tangent space} of $\mathcal{P}_2(\mathcal{Y})$ at $\mu$ is
\begin{equation}
    \mathrm{Tan}_\mu \mathcal{P}_2(\mathcal{Y}) := \overline{\{\nabla \phi : \phi \in C_c^\infty(\mathcal{Y})\}}^{L^2(\mu;\, \mathbb{R}^d)},
\end{equation}
where $C_c^\infty(\mathcal{Y})$ denotes the space of smooth functions with compact support, and the closure is taken in $L^2(\mu; \mathbb{R}^d)$.
\end{restatable}

Results from~\citet{Ambrosio2004-oz} tie these definitions together. First, the metric derivative of an absolutely continuous curve exists almost everywhere and is its minimal speed.

\begin{restatable}[{\citealp[Theorem~1.1.2]{Ambrosio2004-oz}}]{prop}{mdexists}
\label{prop:md_exists}
Let $(\mu_t)_{t \in (a,b)} \in AC^2(a, b; \mathcal{P}_2(\mathcal{Y}))$. Then $|\mu'|(t)$ exists for a.e.\ $t \in (a,b)$, and belongs to $L^2(a, b)$. Moreover, the function $t \mapsto |\mu'|(t)$ satisfies \eqref{eq:w2_bound_curve}, and $|\mu'|(t) \leq m(t)$ for a.e.\ $t$ for any other $m$ satisfying~\eqref{eq:w2_bound_curve}. In this sense, $|\mu'|$ is the smallest admissible speed bound.
\end{restatable}

Second, absolutely continuous curves are exactly those driven by a velocity field of square-integrable norm through the \emph{continuity equation}. Every absolutely continuous curve admits such a field:

\begin{restatable}[{\citealp[Theorem~8.3.1]{Ambrosio2004-oz}}]{theorem}{accurvevelocity}
\label{thm:ac_curve}
Let $(\mu_t)_{t \in (a,b)} \in AC^2(a, b; \mathcal{P}_2(\mathcal{Y}))$. Then there exists a Borel vector field $v\colon (y, t) \mapsto v_t(y)$ such that
\begin{equation}
    v_t \in L^2(\mu_t; \mathbb{R}^d), \qquad \|v_t\|_{L^2(\mu_t)} \leq |\mu'|(t) \quad \text{for a.e.\ } t \in (a,b), \label{eq:vf_lt_metricderivative}
\end{equation}
and the continuity equation
\begin{equation}
    \partial_t \mu_t + \nabla \cdot (v_t \mu_t) = 0 \qquad \text{in } \mathcal{Y} \times (a,b) \label{eq:continuity_equation}
\end{equation}
holds in the distributional sense. Moreover, $v_t \in \mathrm{Tan}_{\mu_t} \mathcal{P}_2(\mathcal{Y})$ for a.e.\ $t \in (a,b)$.
\end{restatable}

Conversely, any curve driven by a velocity field with square-integrable norm is absolutely continuous, and the norm of that field bounds the metric derivative:

\begin{restatable}[{\citealp[Theorem~8.3.1]{Ambrosio2004-oz}}]{prop}{velocityac}
\label{prop:velocity_ac}
Let $(\mu_t)_{t \in (a,b)}$ be a curve in $\mathcal{P}_2(\mathcal{Y})$ that is continuous with respect to $W_2$ and satisfies~\eqref{eq:continuity_equation} for some Borel vector field $v_t$ with $\int_a^b \|v_t\|^2_{L^2(\mu_t)}\, dt < \infty$. Then $(\mu_t) \in AC^2(a, b; \mathcal{P}_2(\mathcal{Y}))$ and
\begin{equation}
    |\mu'|(t) \leq \|v_t\|_{L^2(\mu_t)} \qquad \text{for a.e.\ } t \in (a,b). \label{eq:md_lt_vf}
\end{equation}
\end{restatable}

Third, the velocity field driving a curve is not unique: adding any field $w_t$ with $\nabla \cdot (w_t \mu_t) = 0$ leaves~\eqref{eq:continuity_equation} unchanged. By Proposition~\ref{prop:velocity_ac}, every such field satisfies $\|v_t\|_{L^2(\mu_t)} \geq |\mu'|(t)$ for a.e.\ $t$. The next result characterizes the fields achieving equality: they are exactly the tangent ones, and they are unique.

\begin{restatable}[{\citealp[Proposition~8.4.5]{Ambrosio2004-oz}}]{prop}{uniquetangent}
\label{prop:unique_tangent}
Let $(\mu_t)_{t \in (a,b)} \in AC^2(a, b; \mathcal{P}_2(\mathcal{Y}))$ and let $v_t \in L^2(\mu_t; \mathbb{R}^d)$ be a vector field solving~\eqref{eq:continuity_equation}. Then
\begin{equation}
    \|v_t\|_{L^2(\mu_t)} = |\mu'|(t) \quad \text{for a.e.\ } t \in (a,b)
\end{equation}
if and only if $v_t \in \mathrm{Tan}_{\mu_t}\mathcal{P}_2(\mathcal{Y})$ for a.e.\ $t \in (a,b)$. Moreover, such a $v_t$ is uniquely determined for a.e.\ $t$.
\end{restatable}

We use this characterization in Lemma~\ref{lem:bound_gap} to identify the gap in the upper bound derived in Theorem~\ref{thm:upper_bound}.

\subsection{Metric-derivative bound}
\label{app:metric_bound}

We now apply Theorem~\ref{thm:ac_curve} to the curve traced out on $W_2(\mathcal{Y})$ by perturbing the parameters of the transport map. For a fixed input $x \in \mathcal{X}$ and a direction $h \in \mathbb{R}^p$ in parameter space, we consider the curve
\begin{equation}
    t \mapsto \mu^h_t := T_{\theta + th}(x)_\#\, \nu.
\end{equation}
The main theorem below shows that the curve is absolutely continuous in $W_2(\mathcal{Y})$ under mild regularity on $T$. The metric derivative at $t=0$ measures the local Wasserstein-rate of change of the predicted distribution under the perturbation $h$, which we bound by a quadratic form in $h$ involving the parameter sensitivity matrix $G_\theta(x)$.
\begin{restatable}[Metric-derivative bound]{theorem}{upperbound}
\label{thm:upper_bound}
Fix $x \in \mathcal{X}$ and $h \in \mathbb{R}^p$. Assume there exists an open neighborhood $U \ni \theta$ such that $(\theta', y) \mapsto T_{\theta'}(x)(y)$ is $C^1$ on $U \times \mathcal{Y}$, that $T_{\theta'}(x) : \mathcal{Y} \to \mathcal{Y}$ is a $C^1$ diffeomorphism for every $\theta' \in U$, and that $\int \sup_{\theta' \in U}\|D_\theta T_{\theta'}(x)(y)\|_{\mathrm{op}}^2\, d\nu(y) < \infty$. Then in a neighborhood of $t = 0$, the curve
\begin{equation}
    \mu_t^h := T_{\theta + th}(x)_\#\, \nu
\end{equation}
is absolutely continuous in $W_2(\mathcal{Y})$, and its upper metric derivative at $t = 0$ satisfies
\begin{equation}
    |\mu_h'|^+(0)^2 \leq h^\top G_\theta(x)\, h, \label{eq:metric_bound}
\end{equation}
where $G_\theta(x)$ is given by \eqref{eq:G_def}. In particular, whenever the metric derivative $|\mu_h'|(0)$ exists, \eqref{eq:metric_bound} holds also for $|\mu_h'|(0)$. 
\end{restatable}
\begin{proof}
We drop the dependence on $x$ throughout. Let $U \ni \theta$ be the open neighborhood from the hypotheses. Choose $\delta > 0$ small enough that $\theta + th \in U$ for all $t \in [-\delta, \delta]$. We define
\begin{equation}
    M := \int \sup_{\theta' \in U} \|D_\theta T_{\theta'}(y)\|_{\mathrm{op}}^2\, d\nu(y), \label{eq:M_def}
\end{equation}
which is finite by the integrability assumption.

Consider the candidate vector field
\begin{equation}
    v^h_t(y) := D_\theta T_{\theta + th}\big(T^{-1}_{\theta + th}(y)\big)\, h, \qquad t \in (-\delta, \delta). \label{eq:velocity_field}
\end{equation}
This is well-defined: $T_{\theta + th}$ is a $C^1$ diffeomorphism for $t \in (-\delta, \delta)$, so $T^{-1}_{\theta + th}(y) \in \mathcal{Y}$ exists for every $y \in \mathcal{Y}$, and $D_\theta T_{\theta + th}$ exists at every point of $\mathcal{Y}$ by joint $C^1$ regularity on $U \times \mathcal{Y}$. The curve $(\mu^h_t)$ is continuous with respect to $W_2$, since $W_2(\mu^h_s, \mu^h_t) \leq \|T_{\theta + sh} - T_{\theta + th}\|_{L^2(\nu)} \to 0$ as $s \to t$ and dominated convergence. We verify the remaining hypotheses of Proposition~\ref{prop:velocity_ac}.

\paragraph{Step 1: $v^h_t$ has bounded $L^2(\mu^h_t)$ norm.}
By definition of the $L^2(\mu^h_t)$ norm and the pushforward identity $\mu^h_t = T_{\theta + th}\#\nu$,
\begin{align}
    \|v^h_t\|_{L^2(\mu^h_t)}^2
    &= \int \big\|D_\theta T_{\theta + th}\big(T^{-1}_{\theta + th}(y)\big)\, h\big\|^2\, d\mu^h_t(y) \\
    &= \int \|D_\theta T_{\theta + th}(z)\, h\|^2\, d\nu(z) \label{eq:vt_norm_pre} \\
    &= h^\top G_{\theta + th}(x)\, h, \label{eq:vt_norm}
\end{align}
where the second equality uses the change of variables $z = T^{-1}_{\theta + th}(y)$ (justified by the diffeomorphism assumption), and the third equality applies the definition of $G$ from \eqref{eq:G_def}. Bounding the integrand in \eqref{eq:vt_norm_pre} by $\|h\|^2 \sup_{\theta' \in U} \|D_\theta T_{\theta'}(z)\|_{\mathrm{op}}^2$ and using \eqref{eq:M_def},
\begin{equation}
    \|v^h_t\|_{L^2(\mu^h_t)}^2 \leq \|h\|^2 M. \label{eq:vt_uniform_bound}
\end{equation}
Hence $\int_{-\delta}^{\delta} \|v^h_t\|^2_{L^2(\mu^h_t)}\, dt \leq 2\delta\,\|h\|^2 M < \infty$.

\paragraph{Step 2: $v^h_t$ solves the continuity equation.}
For any test function $\phi \in C_c^\infty(\mathcal{Y})$, the pushforward identity $\mu^h_t = T_{\theta + th}\#\nu$ gives
\begin{equation}
    \frac{d}{dt} \int \phi(y)\, d\mu^h_t(y) = \frac{d}{dt} \int \phi(T_{\theta + th}(z))\, d\nu(z).
\end{equation}
Since $\nabla \phi$ is bounded, the derivative of the integrand is dominated by $\|\nabla \phi\|_\infty \|h\| \sup_{\theta' \in U}\|D_\theta T_{\theta'}(z)\|_{\mathrm{op}}$, which is $\nu$-integrable by the integrability assumption. By dominated convergence, we interchange $d/dt$ and the integral, and apply the chain rule:
\begin{equation}
    \frac{d}{dt} \int \phi(T_{\theta + th}(z))\, d\nu(z) = \int \nabla \phi(T_{\theta + th}(z))^\top D_\theta T_{\theta + th}(z)\, h\, d\nu(z).
\end{equation}
Changing variables back via $y = T_{\theta + th}(z)$ and using \eqref{eq:velocity_field}:
\begin{equation}
    \int \nabla \phi(T_{\theta + th}(z))^\top D_\theta T_{\theta + th}(z)\, h\, d\nu(z) = \int \nabla \phi(y)^\top v^h_t(y)\, d\mu^h_t(y).
\end{equation}
Integration by parts yields:
\begin{equation}
    \int \nabla \phi(y)^\top v^h_t(y)\, d\mu^h_t(y) = -\int \phi(y)\, \nabla \cdot (v^h_t \mu^h_t)(y)\, dy.
\end{equation}
Combining all the equalities, we have
\begin{equation}
    \int \phi\, \partial_t \mu^h_t = -\int \phi\, \nabla \cdot (v^h_t \mu^h_t) \qquad \text{for all } \phi \in C_c^\infty(\mathcal{Y}),
\end{equation}
which is the weak form of the continuity equation
\begin{equation}
    \partial_t \mu^h_t + \nabla \cdot (v^h_t \mu^h_t) = 0 \qquad \text{on } \mathcal{Y} \times (-\delta, \delta). \label{eq:continuity_equation_proof}
\end{equation}

\paragraph{Step 3: Conclusion.}
By Steps 1 and 2, Proposition~\ref{prop:velocity_ac} applies: $(\mu^h_t)_{t \in (-\delta, \delta)}$ is absolutely continuous in $W_2(\mathcal{Y})$, and
\begin{equation}
    |\mu_h'|(r) \leq \|v^h_r\|_{L^2(\mu^h_r)} \qquad \text{for a.e.\ } r \in (-\delta, \delta). \label{eq:metric_deriv_bound}
\end{equation}
Since this bound only holds almost everywhere, we integrate it to obtain the limit superior around zero. By Proposition~\ref{prop:md_exists} and~\eqref{eq:metric_deriv_bound}, for $0 < |s| < \delta$,
\begin{equation}
    W_2(\mu^h_s, \mu^h_0) \leq \left| \int_0^s |\mu_h'|(r)\, dr \right| \leq \left| \int_0^s \|v^h_r\|_{L^2(\mu^h_r)}\, dr \right|.
\end{equation}
By~\eqref{eq:vt_norm}, $\|v^h_r\|_{L^2(\mu^h_r)} = \sqrt{h^\top G_{\theta + rh}(x)\, h}$, which is continuous in $r$: the integrand in~\eqref{eq:vt_norm_pre} is continuous in $r$ by joint $C^1$ regularity and dominated by $\|h\|^2 \sup_{\theta' \in U}\|D_\theta T_{\theta'}(z)\|^2_{\mathrm{op}}$, which is integrable by~\eqref{eq:M_def}. Dividing by $|s|$, the right-hand side converges to $\sqrt{h^\top G_\theta(x)\, h}$ as $s \to 0$. Taking the $\limsup$ of the left-hand side gives
\begin{equation}
    |\mu_h'|^+(0)^2 \leq h^\top G_\theta(x)\, h,
\end{equation}
which is the claimed bound~\eqref{eq:metric_bound}. Finally, since $|\mu_h'|^+(0) = |\mu_h'|(0)$ whenever the latter exists (Definition~\ref{def:metric_derivative}), the bound also holds for $|\mu_h'|(0)$ in that case.

\end{proof}

Theorem~\ref{thm:upper_bound} establishes that the parameter sensitivity matrix $G_\theta(x)$ approximately controls the local Wasserstein-rate of change of the predicted distribution under any parameter perturbation $h$: along the curve $t \mapsto T_{\theta+th}(x)\#\nu$, the metric derivative at $t=0$ is bounded by the quadratic form $\sqrt{h^\top G_\theta(x)\, h}$. 

This recovers the parameter sensitivity matrix that appears in Theorem~\ref{thm:sensitivitybound}, but via a route that is purely geometric: rather than passing through a posterior expectation and Taylor expansion, the bound emerges from the velocity field that the parameter perturbation induces on the Wasserstein manifold. 

\subsection{Gap analysis}
\label{app:gap_analysis}

Theorem~\ref{thm:upper_bound} bounds the upper metric derivative $|\mu_h'|^+(0)$ by $\|h\|_{G_\theta}$, but the inequality is strict in general. We characterize the gap in two cases. First, in dimension one with monotone transport maps, the bound is exact: $G_\theta(x)$ recovers the true metric derivative. Second, in higher dimensions the gap is governed by the orthogonal decomposition of the velocity field into a tangent component and a divergence-free component, with the tangent component achieving equality. The two results are summarized in Lemma~\ref{lem:1d_bound} and Lemma~\ref{lem:bound_gap}, respectively.

\subsubsection{Tightness in dimension one}
\label{app:1d_tightness}
\begin{lemma}[1D tightness]
\label{lem:1d_bound}
Under the hypotheses of Theorem~\ref{thm:upper_bound} with $\dim(\mathcal{Y}) = 1$, assume further that $\nu$ is absolutely continuous with respect to the Lebesgue measure, and that $T_{\theta'}(x)$ monotone in $y$ for every $\theta' \in U$. Then, the inequality \eqref{eq:metric_bound} becomes an equality:
\begin{equation}
    |\mu_h'|(t)^2 = h^\top G_{\theta + th}(x)\, h \qquad \text{for all } t \in (-\delta, \delta).
\end{equation}
\end{lemma}
\begin{proof}
Monotonicity of $T_{\theta+th}(x)$ gives the quantile representation $T_{\theta+th}(x) = F^{-1}_{\theta+th}(x; \cdot) \circ F_\nu$, where $F^{-1}_{\theta+th}(x; \tau)$ is the $\tau$-quantile of $\mu^h_t = T_{\theta+th}(x)\#\nu$.

\textbf{Step 1: metric derivative.} The 1D quantile representation of the squared 2-Wasserstein distance gives, for $s, t \in (-\delta, \delta)$,
\begin{equation}
    W_2^2(\mu^h_{t+s}, \mu^h_t) = \int_0^1 \big(F^{-1}_{\theta + (t+s)h}(x; \tau) - F^{-1}_{\theta + th}(x; \tau)\big)^2\, d\tau.
\end{equation}
Dividing by $s^2$ and passing to the limit $s \to 0$ (interchanging limit and integral by dominated convergence),
\begin{equation}
    |\mu_h'|(t)^2 = \int_0^1 \left|\frac{\partial}{\partial t} F^{-1}_{\theta + th}(x; \tau)\right|^2 d\tau. \label{eq:1d_metric_deriv}
\end{equation}
\textbf{Step 2: $L^2$ norm of the velocity field.} By \eqref{eq:vt_norm} from the proof of Theorem~\ref{thm:upper_bound},
\begin{equation}
    h^\top G_{\theta+th}(x)\, h = \|v^h_t\|_{L^2(\mu^h_t)}^2 = \int \big\|v^h_t(y)\big\|^2\, d\mu^h_t(y). \label{eq:1d_step2_def}
\end{equation}
The velocity field, defined in \eqref{eq:velocity_field}, satisfies
\begin{equation}
    v^h_t(y) = D_\theta T_{\theta+th}(x)(T^{-1}_{\theta+th}(x)(y))\, h = \frac{\partial}{\partial t} T_{\theta+th}(x)(T^{-1}_{\theta+th}(x)(y)), \label{eq:1d_velocity_chain}
\end{equation}
where the second equality applies the chain rule to the parameterization $t \mapsto T_{\theta+th}(x)(z)$ at fixed $z$. Substituting,
\begin{equation}
    \|v^h_t\|_{L^2(\mu^h_t)}^2 = \int \left|\frac{\partial}{\partial t} T_{\theta+th}(x)(T^{-1}_{\theta+th}(x)(y))\right|^2 d\mu^h_t(y). \label{eq:1d_vt_norm_explicit}
\end{equation}
Now we use monotonicity to rewrite this in quantile form. Since $T_{\theta+th}(x) = F^{-1}_{\theta+th}(x; \cdot) \circ F_\nu$, we have $T^{-1}_{\theta+th}(x)(y) = F_\nu^{-1}(F_{\theta+th}(x; y))$, where $F_{\theta+th}(x; \cdot)$ is the CDF of $\mu^h_t$. Combining,
\begin{equation}
    \frac{\partial}{\partial t} T_{\theta+th}(x)(T^{-1}_{\theta+th}(x)(y)) = \frac{\partial}{\partial t} F^{-1}_{\theta+th}(x; F_{\theta+th}(x; y)).
\end{equation}
Applying the change of variables $\tau = F_{\theta+th}(x; y)$ (so $d\mu^h_t(y) = d\tau$) in \eqref{eq:1d_vt_norm_explicit},
\begin{equation}
    \|v^h_t\|_{L^2(\mu^h_t)}^2 = \int_0^1 \left|\frac{\partial}{\partial t} F^{-1}_{\theta + th}(x; \tau)\right|^2 d\tau. \label{eq:1d_vt_norm}
\end{equation}
Comparing \eqref{eq:1d_metric_deriv} and \eqref{eq:1d_vt_norm} gives $|\mu_h'|(t)^2 = h^\top G_{\theta+th}(x)\, h$, which is the claimed equality.
\end{proof}

\subsubsection{Decomposition of the velocity field in higher dimensions}
\label{app:helmholtz}

Let $L^2(\mu^h_t; \mathbb{R}^d)$ denote the Hilbert space of $\mathbb{R}^d$-valued vector fields $v : \mathcal{Y} \to \mathbb{R}^d$ that are square-integrable with respect to $\mu^h_t$, equipped with the inner product $\langle u, v \rangle_{L^2(\mu^h_t)} = \int \langle u(y), v(y) \rangle\, d\mu^h_t(y)$. The tangent space $\mathrm{Tan}_{\mu^h_t} \mathcal{P}_2(\mathcal{Y})$ is a closed subspace of $L^2(\mu^h_t; \mathbb{R}^d)$, and its orthogonal complement consists of the vector fields $w$ that are divergence-free with respect to $\mu^h_t$, i.e., $\nabla \cdot (w\, \mu^h_t) = 0$ in the distributional sense; equivalently, $\int \langle \nabla \phi, w \rangle\, d\mu^h_t = 0$ for every $\phi \in C_c^\infty(\mathcal{Y})$. This is the Helmholtz–Hodge decomposition in $L^2(\mu^h_t)$ \citep{Ambrosio2004-oz, villani2009optimal}. 

We may therefore consider the orthogonal projection operator
\begin{equation}
    \Pi : L^2(\mu^h_t; \mathbb{R}^d) \to \mathrm{Tan}_{\mu^h_t} \mathcal{P}_2(\mathcal{Y}),
\end{equation}
that sends each vector field to its component along $\mathrm{Tan}_{\mu^h_t} \mathcal{P}_2(\mathcal{Y})$. Given that we characterized the tangent space as the closure of gradient fields (Definition~\ref{def:tangent}), the projection $\Pi(v)$ returns the best $L^2(\mu_t^h)$-approximation of $v$ by such fields. The next lemma shows that this projection characterizes the gap in Theorem~\ref{thm:upper_bound}. 

\begin{lemma}[Gap as the divergence-free component]
\label{lem:bound_gap}
Let $v^h_t$ be the velocity field defined in \eqref{eq:velocity_field}, and write
\begin{equation}
    v^h_t = v^{\mathrm{tan}}_t + w_t, \qquad v^{\mathrm{tan}}_t := \Pi(v^h_t), \qquad w_t := v^h_t - v^{\mathrm{tan}}_t,
\end{equation}
so that $v_t^{\tan} \in \mathrm{Tan}_{\mu_t^h}\mathcal{P}_2(\mathcal{Y})$ and $w_t$ is divergence-free with respect to $\mu_t^h$. Then the gap in \eqref{eq:metric_bound} satisfies
\begin{equation}
       h^\top G_{\theta + th}(x) h = \|w_t\|^2_{L^2(\mu^h_t)} +\; |\mu_h'|(t)^2, \qquad a.e.\; t.
\end{equation}
In particular, equality in \eqref{eq:metric_bound} holds if and only if $w_t = 0$, i.e., the velocity field $v^h_t$ lies in the tangent space.
\end{lemma}

\begin{proof}
By orthogonality of the decomposition $v^h_t = v^{\mathrm{tan}}_t + w_t$ in $L^2(\mu^h_t)$,
\begin{equation}
    \|v^h_t\|^2_{L^2(\mu^h_t)} = \|v^{\mathrm{tan}}_t\|^2_{L^2(\mu^h_t)} + \|w_t\|^2_{L^2(\mu^h_t)}. \label{eq:helmholtz_pythagoras}
\end{equation}
$v_t^{\tan}$ still solves the continuity equation, as $\nabla \cdot (\mu_t w_t) = 0$. By Theorem~\ref{prop:unique_tangent}, the metric derivative coincides with the $L^2(\mu^h_t)$-norm of the tangent component:
\begin{equation}
    |\mu_h'|(t)^2 = \|v^{\mathrm{tan}}_t\|^2_{L^2(\mu^h_t)}. \label{eq:tan_is_metric_deriv}
\end{equation}
Combining \eqref{eq:helmholtz_pythagoras} and \eqref{eq:tan_is_metric_deriv}, and using $\|v^h_t\|^2_{L^2(\mu^h_t)} = h^\top G_{\theta+th}(x)\, h$ from \eqref{eq:vt_norm},
\begin{equation}
    h^\top G_{\theta + th}(x)\, h  = |\mu_h'|(t)^2 + \|w_t\|^2_{L^2(\mu^h_t)}.
\end{equation}
The "if and only if" claim follows since $\|w_t\|^2 \geq 0$ with equality iff $w_t = 0$.
\end{proof}

Lemma~\ref{lem:bound_gap} suggests a refinement of $G_\theta(x)$: replacing $v^h_0$ with its tangent projection $\Pi(v^h_0)$ gives a quadratic form $h^\top \tilde G_\theta(x)\, h := \|\Pi(v^h_0)\|^2_{L^2(\mu^h_0)}$ that exactly recovers the squared metric derivative $|\mu_h'|(0)^2$. We leave this as a future research direction. 

\clearpage

\section{Exact metric tensor under the Knothe–Rosenblatt distance}
\label{app:KR}

Theorem~\ref{thm:upper_bound} bounds the metric derivative on the Wasserstein-2 manifold of the probability curve $t \mapsto \mu_{\theta + th}(x)$, evaluated at $t=0$ by the quadratic form $h^\top G_\theta(x) h$. We now show that for lower-triangular affine transport maps (e.g., obtained via Cholesky parameterization), this bound is exact under the Knothe--Rosenblatt (KR) distance~\citep{rosenblatt1952remarks, knothe1957contributions}. In this setting, $G_\theta(x)$ is the exact metric tensor, suggesting that the parameter sensitivity matrix has a precise geometric meaning beyond the loose Wasserstein-2 upper bound.

\subsection{Background: KR rearrangement, KR distance, and metric derivative}

\begin{definition}[Knothe--Rosenblatt rearrangement]
\label{def:KR_rearrangement}
Let $\nu, \mu$ be probability measures on $\mathbb{R}^d$. A \emph{Knothe--Rosenblatt rearrangement} from $\nu$ to $\mu$ is a triangular map $T^{\mathrm{KR}}: \mathbb{R}^d \to \mathbb{R}^d$ that is component-wise monotone in its arguments and satisfies $T^{\mathrm{KR}}_\# \nu = \mu$.
\end{definition}
When $\nu$ is absolutely continuous with respect to the Lebesgue measure, the KR rearrangement from $\nu$ to any target $\mu$ exists and is $\nu$-a.e. unique~\citep{rosenblatt1952remarks, knothe1957contributions}. 

\begin{definition}[Knothe--Rosenblatt distance]
\label{def:KR_distance}
Let $\nu$ be a probability measure on $\mathbb{R}^d$ absolutely continuous with respect to the Lebesgue measure, and let $\mu_1, \mu_2$ be probability measures on $\mathbb{R}^d$ with finite second moments. Let $T_1^{\mathrm{KR}}, T_2^{\mathrm{KR}}$ be the KR rearrangements from $\nu$ to $\mu_1, \mu_2$ respectively. The squared \emph{Knothe--Rosenblatt distance} between $\mu_1$ and $\mu_2$ is
\begin{equation}
    W_{\mathrm{KR}}^2(\mu_1, \mu_2) := \int \| T_1^{\mathrm{KR}}(y) - T_2^{\mathrm{KR}}(y) \|_2^2 \, d\nu(y).
\end{equation}
\end{definition}

\begin{definition}[Metric derivative under $W_{\mathrm{KR}}$]
\label{def:KR_metric_derivative}
For the curve $t \mapsto \mu_\theta^t$ along a perturbation direction $h \in \mathbb{R}^p$, the KR metric derivative at $t = 0$ is
\begin{equation}
    |\mu'_h|_{\mathrm{KR}}(0) := \lim_{t \to 0} \frac{W_{\mathrm{KR}}( \mu_\theta^t,\,  \mu_\theta^0)}{|t|},
\end{equation}
when the limit exists.
\end{definition}

\subsection{KR metric derivative for lower-triangular affine pushforwards}

We first specialize the KR distance to the family of lower-triangular affine pushforwards relevant to our setting.

\begin{lemma}[KR distance for lower-triangular affine pushforwards]
\label{lem:KR_affine}
Let $\nu$ be a probability measure on $\mathbb{R}^d$ absolutely continuous with respect to the Lebesgue measure, with finite second moments. Let $T_1, T_2$ be lower-triangular affine maps $T_i(y) = b_i + L_i y$ with $L_i$ having strictly positive diagonals. Then $T_i$ is the KR rearrangement from $\nu$ to $T_{i\#}\nu$, and
\begin{equation}
    W_{\mathrm{KR}}^2(T_{1\#}\nu,\, T_{2\#}\nu) = \int \| T_1(y) - T_2(y) \|_2^2 \, d\nu(y).
\end{equation}
\end{lemma}

\begin{proof}
By the discussion following Definition~\ref{def:KR_rearrangement}, $T_i$ is triangular and component-wise monotone (since $L_i$ has positive diagonal entries), hence $T_i$ is a KR rearrangement from $\nu$ to $T_{i\#}\nu$. By the $\nu$-a.e. uniqueness of the KR rearrangement under the absolute continuity assumption on $\nu$, $T_i$ is the unique such map. Substituting into Definition~\ref{def:KR_distance} gives the claimed formula.
\end{proof}

We now show that in contrast to the Wasserstein-2 bound (Theorem~\ref{thm:upper_bound}), $G_\theta$ exactly captures the squared metric derivative under the KR distance for lower-triangular affine pushforwards.

\begin{theorem}[Exact metric tensor for lower-triangular affine maps]
\label{thm:KR_exact}
Let $T_\theta(y) = b_\theta + L_\theta y$ be a lower-triangular affine transport map with $L_\theta$ having strictly positive diagonals, and assume $\theta \mapsto (b_\theta, L_\theta)$ is differentiable. Assume an absolutely continuous measure $\nu$ has zero mean and identity covariance. Then for any perturbation direction $h \in \mathbb{R}^p$,
\begin{equation}
    |\mu'(\theta, h)|_{\mathrm{KR}}^2 = h^\top G_\theta\, h, \quad \text{where} \quad G_\theta = \int D_\theta T_\theta(y)^\top D_\theta T_\theta(y) \, d\nu(y).
\end{equation}
\end{theorem}

\begin{proof}
By Lemma~\ref{lem:KR_affine}, for any two affine maps $T_i(y) = b_i + L_i y$:
\begin{equation}
    W_{\mathrm{KR}}^2(T_{1\#}\nu, T_{2\#}\nu) = \int \|T_1(y) - T_2(y)\|^2 \, d\nu(y).
\end{equation}
Expanding $T_1(y) - T_2(y) = (b_1 - b_2) + (L_1 - L_2) y$ and using $\mathbb{E}_\nu[y] = 0$ and $\mathbb{E}_\nu[y y^\top] = I$, the cross-term vanishes and the integral evaluates to
\begin{equation}
\label{eq:KR_affine_closed}
    W_{\mathrm{KR}}^2(T_{1\#}\nu, T_{2\#}\nu) = \|b_1 - b_2\|^2 + \|L_1 - L_2\|_F^2.
\end{equation}
Applying \eqref{eq:KR_affine_closed} to the parameter curve $T_{\theta+th}$ vs.\ $T_\theta$,
\begin{equation}
    W_{\mathrm{KR}}^2(\mu(\theta + th), \mu(\theta)) = \|b_{\theta+th} - b_\theta\|^2 + \|L_{\theta+th} - L_\theta\|_F^2.
\end{equation}
A first-order Taylor expansion of $b_\theta$ and $L_\theta$ in $\theta$ gives, as $t \to 0$,
\begin{align}
    \|b_{\theta+th} - b_\theta\|^2 &= t^2 \|D_\theta b_\theta\, h\|^2 + o(t^2), \\
    \|L_{\theta+th} - L_\theta\|_F^2 &= t^2 \|D_\theta L_\theta\, h\|_F^2 + o(t^2).
\end{align}
Dividing by $t^2$ and taking the limit,
\begin{equation}
    |\mu'(\theta, h)|_{\mathrm{KR}}^2 = \|D_\theta b_\theta\, h\|^2 + \|D_\theta L_\theta\, h\|_F^2.
\end{equation}
By Lemma~\ref{lem:kernel_decomposition}, this equals $h^\top G_\theta h$ where $G_\theta = \int D_\theta T_\theta(y)^\top D_\theta T_\theta(y)\, d\nu(y)$.
\end{proof}

\clearpage
\section{Experimental details}
\label{app:exp}
\subsection{Conditional Histogram Experimental Details}
\label{app:condhistdetails}

\subsubsection{Data Simulation}
The input space is $\mathcal{X} = [0,1]$. For each $x$, the conditional distribution $\mu_x$ is a symmetric two-component Gaussian mixture
\begin{align}
    \mu_x = \frac{1}{2}\,\mathcal{N}\left(m(x) - d(x),\;\sigma^2(x)\right)
          + \frac{1}{2}\,\mathcal{N}\left(m(x) + d(x),\;\sigma^2(x)\right),
    \label{eq:mixture}
\end{align}
with conditional mean $m(x) = x$, mode half-separation
$d(x) = \exp\left(-20\,(x - 0.5)^2\right)$, and component standard
deviation $\sigma(x) = 0.2\,(1+x)^2$.

This design is chosen so that $m(x)$ is smooth while the quantile functions are localized. With this setting, we anticipate that a mean-based active learning method will sample the space evenly while a quantile-based active learning method will be sensitive to areas of $\mathcal{X}$ where the variance changes.

\subsubsection{Gaussian Process Regression}
We fit a quantile function $x \mapsto F_x^{-1}(\tau_j)$ with $F_x^{-1}(\tau_j) = m_\theta(x) + g_\theta(x; \tau_j)$ for each quantile $\tau_j \in \{\tau_1, \dots, \tau_M\}$ on a uniform grid ($M = 30$, splitting $[0.03, 0.97]$ for stability, boundaries included) independently with a Gaussian process. We estimate posterior expected variance in \eqref{eq:shape_risk_1d} by integrating the variance across the quantiles. For the mean strategy, we fit a Gaussian process predicting the mean and obtain the posterior expected variance. Kernel hyperparameters (amplitude and bandwidth of an RBF kernel) are pre-fitted by maximum likelihood on a dense ground-truth grid and then fixed throughout the active learning loop; this isolates the effect of the acquisition strategy from hyperparameter estimation artifacts.

\subsubsection{Active Learning Strategies}

\paragraph{Wasserstein AL} Select
\begin{align}
    x_{\mathrm{new}} &= {\arg\min}_{x' \in\mathcal{X}_{\mathrm{cand}}}
    \; \mathbb{E}_{x^* \sim p_{\mathrm{test}}}\int_0^1
    \operatorname{Var}\left[F_{x^*}^{-1}(\tau) \mid
    \mathcal{D} \cup \{x'\}\right]\,d\tau,
\end{align}
    targeting the full distributional risk.
\paragraph{Mean AL} Select
    \begin{align}
    x_{\mathrm{new}} = \arg\min_{x' \in \mathcal{X}_{\mathrm{cand}}}
    \; \mathbb{E}_{x^* \sim p_{\mathrm{test}}}
    \operatorname{Var}\!\left[m(x^*) \mid
    \mathcal{D} \cup \{x'\}\right],
    \end{align}
    targeting only the first term (mean) of the risk decomposition.
\paragraph{Random } Select $x_{\mathrm{new}}$ uniformly at random.

\subsubsection{Experiments}
Each strategy starts from $N_{\mathrm{init}} = 5$ shared initial observations (including the boundary points $x = 0.005$ and $x = 0.995$) and acquires $5$ additional points sequentially. 

Mean Wasserstein loss \eqref{eq:quantile_w2}, mean squared error between the predicted mean $m_\theta(x) = \int F^{-1}_\theta(x)d\tau$ and the true mean, and the mean shape error \eqref{eq:1d_shape_loss} across conditions are reported.

\subsection{Baselines}
\label{app:ntk_from_pushforward}

The NTK baseline benchmarked throughout Section~\ref{sec:experiments} is the standard finite-width NTK \eqref{eq:ntk} for last-layer parameters, applied to the mean-fitting network that is trained to predict the sample mean. The size of the mean-fitting network is set as identical to the transport map trained. The NTK kernels are paired with the same acquisition rule \eqref{eq:acquisition} as tNTK.

For the conditional quantile prediction, we have two additional baselines that use NTK of networks that each predict mean and variance, and quartiles. We also apply BALSA\textsuperscript{EMD}from \citet{Werner2025-qo}, which is the only work that takes the distribution into account to the best of our knowledge. BALSA\textsuperscript{EMD} was benchmarked only for conditional quantile prediction, as it is only designed for 1-dimensional distributions.

\subsection{Affine transport of 2D Gaussian} \label{app:exp_2d_gaussian}
\label{app:affine-kernel}
 
In this section, we derive a fully analytic expression for tNTK when the conditional generative model is a two-dimensional affine map whose covariance factor is parameterized via a Cholesky decomposition, along with the experimental details for the results presented in Section \ref{sec:result_affine_2d_gaussian}. The resulting kernel requires no Monte Carlo integration over the base measure and no automatic differentiation---only inner products of intermediate features and
pointwise sigmoid evaluations. We then describe the details about training and data.
 

\subsubsection{Affine Cholesky Parameterization}
\label{app:affine-param}
 
We first show the general result on the affine map with Cholesky decomposition on a $d$-dimensional Gaussian, with experimental details deferred to Appendix~\ref{app:exp_2d_gaussian_training}. Let $x \in \mathcal{X}$ denote the conditioning variable and
$y = (y_1, \dots, y_d)^\top \sim \nu$ a sample from the base measure $\nu= \mathcal{N}(0, I_d)$.  We consider a conditional transport map of the form
\begin{equation}
  T_\theta(x)(y) \;=\; b_\theta(x) \;+\; L_\theta(x)\, y \;\in\; \mathbb{R}^d,
\end{equation}
where $b_\theta(x) \in \mathbb{R}^d$ is the predicted mean and
$L_\theta(x) \in \mathbb{R}^{d \times d}$ is a lower-triangular matrix with strictly
positive diagonal entries, so that
$\Sigma_\theta(x) = L_\theta(x)\,L_\theta(x)^\top$ is a valid covariance matrix.  The
pushforward distribution is
$\mu(\theta, x) = \mathcal{N}\!\bigl(b_\theta(x),\, \Sigma_\theta(x)\bigr)$.
 
Both $b_\theta(x)$ and $L_\theta(x)$ are produced by a neural network with the following
last-layer structure.
 
\paragraph{Backbones.}
Two feature extractors (which may or may not share parameters) map the input to
representations
\begin{equation}
  h_\theta(x) \;\in\; \mathbb{R}^H, \qquad g_\theta(x) \;\in\; \mathbb{R}^H,
\end{equation}
where $h_\theta$ feeds the mean head and $g_\theta$ feeds the Cholesky head.  Their
parameters are part of $\theta$ but are held fixed when defining the \emph{last-layer}
kernel.

\paragraph{Mean head.}
A linear layer with weight rows $m_j^\top, \in \mathbb{R}^H,\; j = 1, \dots, d$ produces the mean vector:
\begin{equation}
  \label{eq:mean-head}
  b^j_\theta(x)
  \;=\;
    m_j^\top h_\theta(x) + b_j.
\end{equation}
 
\paragraph{Cholesky head.}
A linear layer with weight rows
$\ell_{ji}^\top \in \mathbb{R}^H$ and biases
$k_{ji} \in \mathbb{R}$ for $j = 1, \dots, d,\; i = 1, \dots, j$ produces the elements of the lower-triangular factor, where the softplus function $\sigma(t) = \log(1 + e^t)$ is applied to the diagonal entries ($i = j$).

\begin{equation}
  L_\theta(x)
  \;=\;
  \begin{bmatrix}
    \sigma\!\bigl(\ell_{11}^\top g_\theta(x) + k_{11}\bigr) & 0 & \dots & 0 \\[5pt]
    \ell_{21}^\top g_\theta(x) + k_{21} & \sigma\!\bigl(\ell_{22}^\top g_\theta(x) + k_{22}\bigr) & \dots & 0 \\
    \vdots & & & \vdots \\
    \ell_{d1}^\top g_\theta(x) + k_{d1} & \dots & & \sigma\!\bigl(\ell_{dd}^\top g_\theta(x) + k_{dd}\bigr)
  \end{bmatrix}.
\end{equation}

\paragraph{Last-layer gradients.} Conditional transport can be written as
\begin{equation}
    T_\theta^j(x)(y) = m_j^\top h(x) + b_j + \sum_{i=1}^{j-1} (\ell_{ji}^\top g_\theta(x) + k_{ji}) y_i + \sigma\left(l_{jj}^\top g_\theta(x) + k_{jj} \right)y_j
\end{equation}
where $T_\theta^j(x)(y)$ denotes $j$-th dimension of the output.
We compute the gradient of $T_\theta^j(x)(y)$ with respect to the last-layer parameters $\theta=\{m, b,\ell, k\}$. Let $\sigma'$ denote the sigmoid function (the derivative of softplus).
\begin{align}
    & \nabla_{m_j} T_\theta^j(x)(y) = h_\theta(x) \\
    & \nabla_{b_j} T_\theta^j(x)(y) = 1 \\
    & \nabla_{\ell_{ji}} T_\theta^j(x)(y) = g_\theta(x) y_i, \quad i = 1,\dots ,j-1 \\
    & \nabla_{k_{ji}} T_\theta^j(x)(y) = y_i, \quad i=1, \dots, j-1 \\
    & \nabla_{\ell_{jj}} T_\theta^j(x)(y) = \sigma'\left(\ell_{jj}^\top g_\theta(x) + k_{jj} \right)y_jg_\theta(x)\\
    & \nabla_{k_{jj}} T_\theta^j(x)(y) = \sigma'\left(\ell_{jj}^\top g_\theta(x) + k_{jj} \right)y_j\\
\end{align}
and we have the closed-form tNTK:
\begin{align}
  k_\theta(x, x')
  \;&=\; \int \operatorname{Tr}\!\bigl[D_w T_\theta(x)(y)^\top\, D_w T_\theta(x')(y)\bigr]
  \,d\nu(y), \\
  &= \int \sum_{j=1}^d \left(h_\theta(x)^\top h_\theta(x') + 1 + \sum_{i=1}^{j-1}\left(g_\theta(x)^\top g_\theta(x') + 1 \right)y_i^2 \right.\nonumber \\
  & \quad \left. + \sigma'(\ell_{jj}^\top g_\theta(x) + k_{jj})\sigma'(\ell_{jj}^\top g_\theta(x') + k_{jj})\left(g_\theta(x)^\top g_\theta(x') + 1\right) y_j^2 \right) d\nu(y) \\
  &= d\left(h_\theta^\top h_\theta(x') + 1 \right) + \\ &\qquad\left(\left(\frac{d(d-1)}{2} + \sum_j \sigma'(\ell_{jj}^\top g_\theta(x) + k_{jj} )\sigma'(\ell_{jj}^\top g_\theta(x') + k_{jj} )\right)(g_\theta(x)^\top g_\theta(x') + 1)\right)
\end{align}

\subsubsection{Training}
\label{app:exp_2d_gaussian_training}
We use a training pool of 200 conditions and 100 test conditions, with active learning batches of size $B = 2$ over $5$ rounds starting from $B_0 = 5$ initial conditions. The mean and covariance heads are separate MLPs, each with 2 hidden layers and ReLU activations followed by a linear output layer of dimension 2 (mean) and 3 (Cholesky factor entries), respectively. We train both networks jointly using the analytical KL divergence between the predicted and ground-truth Gaussian as the loss:
\begin{align}
D_\mathrm{KL}\!\left(\mathcal{N}(\mu_t, \Sigma_t) \,\|\, \mathcal{N}(\mu_p, \Sigma_p)\right)
= \frac{1}{2}\left[
    \operatorname{tr}(\Sigma_p^{-1}\Sigma_t)
    + (\mu_p - \mu_t)^\top \Sigma_p^{-1} (\mu_p - \mu_t)
    - d
    + \log\frac{\det \Sigma_p}{\det \Sigma_t}
\right] \\
= \frac{1}{2}\left[
    \| L_p^{-1} L_t \|_F^2
    + \| L_p^{-1}(\mu_p - \mu_t) \|_2^2
    - d
    + 2\sum_{i=1}^{d} \log (L_p)_{ii}
    - 2\sum_{i=1}^{d} \log (L_t)_{ii}
\right].
\end{align}
The loss is optimized with Adam using a learning rate of $0.001$ for 2000 steps. For the evaluation, the Wasserstein-2 distance between predicted and ground truth Gaussian distributions was calculated with its closed form~\citep{MR752258}.

\subsubsection{Data}
Data is generated with condition covariates $x = (x_1, x_2)$ that specifies $b_\theta(x)$ and $\Sigma_\theta(x)$:
\begin{align}
    b_{\theta}(x) &= 2.0 [e^{-6x_1 x_2}, e^{-6x_1(1-x_2)}]^\top \\
    \Sigma_\theta(x) &= \begin{bmatrix}
\lambda_1^2 \cos^2\varphi + \lambda_2^2 \sin^2\varphi & (\lambda_1^2 - \lambda_2^2)\sin\varphi\cos\varphi \ \\
(\lambda_1^2 - \lambda_2^2)\sin\varphi\cos\varphi & \lambda_1^2 \sin^2\varphi + \lambda_2^2 \cos^2\varphi
\end{bmatrix}
\end{align}
with $\varphi = \pi x_2$, $\lambda_1 = 0.2 + 2 e^{-10(1 - x_1)}$, and $\lambda_2 = 0.2$. We construct $196$ training conditions on a $14 \times 14$ Cartesian grid in $[0, 1]^2$, and $100$ test conditions on a $10 \times 10$ grid uniformly spaced in $[1/28,\, 1 - 1/28]^2$ to avoid overlap with the training grid. The Cholesky factor $L_\theta(x)$ is computed from $\Sigma_\theta(x)$.

\subsection{Conditional quantile prediction}
\label{app:exp_quantile}
\subsubsection{Model and training}
We train MLPs that produce 32 quantile values per condition following \citet{Moon2021-mu}. The MLP outputs a vector $z \in \mathbb{R}^{H}$ from its last hidden layer. Two linear heads then produce: (i) the smallest predicted quantile value $\hat{F}^{-1}_{\theta}(\tau_1) = h_{\min}^\top z + b_{\min} \in \mathbb{R}$, and (ii) $31$ non-negative increments $\delta = \mathrm{softplus}(W z + b) \in \mathbb{R}^{31}_{\geq 0}$. The full vector of $32$ predicted quantiles is then
\begin{equation}
    \hat{F}^{-1}_{\theta}(\tau_k) = \hat{F}^{-1}_{\theta}(\tau_1) + \sum_{j=1}^{k-1} \delta_j, \quad k = 1, \ldots, 32,
\end{equation}
ensuring monotonicity by construction (\emph{non-crossing quantiles}). The targeted quantile levels are $\tau_k = (2k - 1)/64$ for $k = 1, \ldots, 32$, i.e., $32$ uniformly spaced points in $[1/64, 1 - 1/64]$. Empirical quantiles at these levels are computed from the per-condition samples and used as training targets.

The model is trained with squared Wasserstein-2 loss with the Adam optimizer for 200 minimum and 3000 maximum epochs with early stopping patience of 10 with validation loss. Validation loss is calculated as the loss with empirical quantile values calculated from 20\% of held-out per-condition validation samples. 

For each dataset, we select the model configuration and weight decay by 5-fold cross-validation over $10$ conditions, choosing the combination that minimizes the squared Wasserstein-2 distance on held-out conditions. The search grid is:
\begin{itemize}
    \item Model configurations $(\text{width}, \text{depth})$: $(64, 2)$, $(32, 3)$, $(128, 3)$.
    \item Weight decay: $\{0.1, 0.01, 10^{-3}, 10^{-4}\}$.
\end{itemize}
The selected hyperparameters are reported in Table~\ref{tab:quantile_hparams}.

\begin{table}[h]
\centering
\caption{Selected hyperparameters for the conditional quantile network, per dataset.}
\label{tab:quantile_hparams}
\begin{tabular}{lcccc}
\toprule
Dataset & $B_0$ & $(\text{width}, \text{depth})$ & \# bins & Weight decay \\
\midrule
\multirow{2}{*}{Simulated Gaussian mixture} & 5  & $(128, 3)$ & 32 & $10^{-4}$ \\
& 10 & $(32, 3)$  & 32 & $10^{-3}$ \\
\midrule
\multirow{2}{*}{Temperature}                & 5  & $(64, 2)$  & 32 & $10^{-3}$ \\
& 10 & $(32, 3)$  & 32 & $10^{-2}$ \\
\midrule
\multirow{2}{*}{Dew Point}                  & 5  & $(64, 2)$  & 32 & $10^{-4}$ \\
& 10 & $(64, 2)$  & 32 & $10^{-1}$ \\
\bottomrule
\end{tabular}
\end{table}

Active learning loops of iterations 0-5 over 5 folds are run on CPU with 10 GB memory within 2 hours on an internal cluster.

\subsubsection{Data}
\paragraph{Simulated 1D mixture Gaussian}
We simulate 1000 1-dimensional Gaussian mixture distribution with 3 components, where parameters are randomly selected with $\mu_i \sim \text{Unif}(-5, 5), \sigma_i \sim \text{Unif}(1.0, 1.5) \ , i \in (1, 2, 3)$ and $\pi \sim \text{Dirichlet}(0.25, 0.25, 0.25)$. 600 samples are generated from each distribution. Condition embedding of $(\mu_1, \sigma_1, \pi_1, \dots, \mu_3, \sigma_3, \pi_3)$ is generated by ordering components so that $\mu_1 \leq \mu_2 \leq \mu_3$ for identifiability.

\paragraph{Temperature}
We take the daily surface temperature measurements between 2023 and 2025 measured across 11396 stations from \citet{noaa_gsod_1999}(CC0-1.0 license). 1000 stations with no less than 256 days of measurements, with non-NaN station locations, and with the highest Wasserstein distance from the control distribution of a random Gaussian with mean and standard deviation of all measurements were selected. The distance is calculated from 256 random samples from each station and 1000 random samples from the control distribution. The resulting station had 967-971 days of measurements. $(\text{longitude}, \text{latitude}, \text{elevation})$ was used as 3-dimensional condition embedding for each station. Yearly temperature and each of the condition embedding dimensions were globally scaled to unit variance and $0$ mean before training.

\paragraph{Dew point}
We take the daily dew point measurements from the same data source as in the Temperature dataset. 1000 stations with no less than 512 days of measurements and with non-NaN station locations were randomly selected, resulting in 515-970 days of measurements. The same condition embedding and data scaling were used as in the Temperature dataset.

\subsection{Conditional flow matching}
\label{app:exp_cflow}
\subsubsection{Model and training}
We adapted the source code of \citet{Klein2025-qb} \footnote{\url{https://github.com/theislab/CellFlow}, \texttt{v0.0.6}, MIT license} with linear interpolant:
\begin{align}
    \phi^{0, x}(y_i) &= y \sim \mathcal{N}(0, I)  \\
    u_i^{t, x}(\phi^{t, x}(y_i)) &= \frac{d}{dt}\phi^{t, x}(y_i) = \frac{y^{x}_i-\phi^{t, x}(y_i)}{1-t}
\end{align}
where $y_i^x$ is a sample of condition $x$. We learn $v_\theta(t, x, \phi^{t, x}(y_i))$ to approximate $u^{t, x}_i(\phi^{t, x}(y_i))$~\citep{Lipman2022-mt}.
The flow ODE is parameterised by a velocity field
\begin{equation}
    v_\theta(t, x, \phi^{t, x}(y_i)) = g_\theta\!\left(\text{concat}\bigl(
        f_\theta^{\text{time}}(t),\;
        f_\theta^{\text{sample}}(\phi^{t, x}(y_i)),\;
        f_\theta^{\text{cond}}(x),
    \bigr)\right),
\end{equation}
where $f_\theta^{\text{sample}}$, $f_\theta^{\text{cond}}$, and $g_\theta$ are MLPs with SiLU activation~\citep{Elfwing2017-vk}, $f_\theta^{\text{time}}$ is a Fourier-feature time encoder, and the condition embedding is mean-pooled before $f_\theta^{\text{cond}}$. A combination of dropout rates among $(0.0, 0.5)$ and one of three model configurations in Table \ref{tab:cellflow_configs} for each dataset is selected to be the one with the least Sinkhorn divergence in validation conditions from the 5-fold CV of 10 randomly selected conditions. 

\begin{table}[h]
\centering
\caption{CellFlow model configurations}
\label{tab:cellflow_configs}
\begin{tabular}{lccccc}
\toprule
Config & $f_\theta^{\text{sample}}$ & $f_\theta^{\text{cond}}$ & $g_\theta$ & Time frequencies & $f_\theta^{\text{time}}$\\
\midrule
\texttt{1}  & $[64,64]$       & $[64,64]$       & $[128,128,128,32]$  & 32 & [128, 32] \\
\texttt{2} & $[64,64]$       & $[64,64]$       & $[128,128,128,128]$ & 128 & [128, 128] \\
\texttt{3} & $[256,256]$     & $[64,64]$       & $[512,512,512,64]$  & 64 & [512, 64] \\
\bottomrule
\end{tabular}
\end{table}

For a given set of training conditions, we hold out 20\% of samples per condition as validation samples. At each training step, a $256$-cell batch is sampled per condition across $128$ time samples of $t$; one optimization step processes all training conditions jointly. The model is trained with Adam (per-dataset learning rate, Table~\ref{tab:cellflow_hparams}) for up to $n_{\text{eval}}$ steps; validation is run every $1000$ steps on Sinkhorn divergence ($\epsilon = 0.1\sigma$, where $\sigma$ is the standard deviation of cost matrix entries) over the held-out validation samples. We additionally apply a plateau-based learning-rate scheduler (patience $5$, factor $0.5$, $\text{lr}_{\min}=0.1\text{lr}$) to CLINC 150 dataset and early stopping with validation loss. Model parameters from the best validation step are restored when early stopping is triggered before kernel computation and test evaluation.

\begin{table}[h]
\centering
\caption{Per-dataset CellFlow hyperparameters}
\label{tab:cellflow_hparams}
\begin{tabular}{lccccc}
\toprule
Dataset & lr & Config & dropout rate & $n_{\text{eval}}$ & early-stopping patience \\
\midrule
Radial flow             & $5\!\times\!10^{-6}$ & 1 & 0   & $30{,}000$  & 5  \\
Van der Pol oscillation & $5\!\times\!10^{-5}$ & 1 & 0   & $30{,}000$  & 5  \\
CLINC 150               & $5\!\times\!10^{-5}$ & 1 & 0.5 & $300{,}000$ & 30 \\
\bottomrule
\end{tabular}
\end{table}


Active learning uses 5-fold cross-validation over conditions: each fold holds out $1/5$ of conditions for test evaluation, and the remaining train/validation conditions are revealed in batches across 3 AL iterations. Iteration $0$ uses random initial conditions and is therefore strategy-independent; we train it once per (fold, seed)  to eliminate GPU non-determinism drift across strategies.

Sinkhorn divergence for validation samples and test conditions is evaluated by integrating the flow ODE with adaptive Tsit5($\text{rtol}=\text{atol}=10^{-5}$) from 256 standard-normal source samples (sampled with replacement if the number of available samples per condition is less than 256) to $t=1$, with $\epsilon = 0.05\sigma$, where $\sigma$ is the standard deviation of cost matrix entries.

For the kernel calculation, $M=4096$ random samples from the standard normal distribution are integrated with the flow ODE with adaptive Tsit5($\text{rtol}=\text{atol}=10^{-5}$). Active learning loops of iterations 0-3 over 5 folds are run on L40S GPUs with 30GB of memory within 48 hours per seed on an internal cluster.

\subsubsection{Data}
\paragraph{Simulated radial flow}
We simulated radial flow in $\mathds{R}^2$ with a global attractor of concentric circles centered at the origin with radius $\sqrt{\alpha/\beta}$, and the steepness of the attraction is controlled by $\alpha$, where $y \sim \mathcal{N}(0, I)$.
\begin{equation}
    \frac{d}{dt}\phi^{t, x}(y) = (\alpha - \beta ||\phi^{t, x}(y)||^2)\phi^{t, x}(y), \ \ \alpha, \beta > 0
\end{equation}
The pool comprises 100 conditions: 50 sampled with jitter around a nominal point, $\alpha \sim \mathrm{Unif}(0.98, 1.02)$ and $\beta \sim \mathrm{Unif}(2.48, 2.52)$, and 50 from the Cartesian product of $\alpha \in \{0.1, 0.31, 0.52, \ldots, 2.0\}$ ($10$ uniformly spaced points in $[0.1, 2.0]$) and $\beta \in \{0.5, 1.625, 2.75, 3.875, 5.0\}$ ($5$ uniformly spaced points in $[0.5, 5.0]$). For each condition, we generate $1000$ particle trajectories by Euler integration of the radial flow ODE for $100$ steps with $\Delta t = 0.01$, yielding the empirical distribution at $t = 1$. The conditioning input is $(\alpha, \beta) \in \mathbb{R}^2$.

\paragraph{Simulated Van der Pol oscillation}
We simulated Van der Pol oscillation for varying starting particle locations $y \sim \mathcal{N}([c_1, c_2]^\top, \sigma^2I)$ and parameter $\mu$.
\begin{align}
\frac{dx_1}{dt} &= x_2 \\
\frac{dx_2}{dt} &= \mu(1 - x_1^2)\, x_2 - x_1
\end{align}

Random direction and radius in the 4-sphere were sampled, where 
$r  = \min(\text{Exp}(2.5), 1)$, $\text{Exp}$ denotes samples from exponential distribution, $u$ is the uniformly randomly selected unit vector in 4-sphere, and $h$ is the vector of half ranges for each parameter.
\begin{align}
\text{Parameter:} \quad
&\theta = (\mu, c_1, c_2, \sigma) = \text{clip}(\theta_0 + r(u \odot h), \theta_{\min}, \theta_{\max})\\
&\theta_0 = (\mu_0, c_{1,0}, c_{2,0}, \sigma) = (2.0, 0.0, 0.0, 0.5) \\
&\theta_{\min} = (-0.5, -1.5, -1.5, 0.2) \\
&\theta_{\max} = (3.0, 1.5, 1.5, 1.0) \\
&h_i = \max(\theta_{0,i} -\theta_{\min,i}, \theta_{\max, i} - \theta_{0, i})
\end{align}
$\phi^{3,x}(y)$ for 1000 particles per condition is calculated with Euler integration with 300 steps ($\Delta t$ = 0.01). $\theta$ was used as the condition embedding.

\paragraph{CLINC 150} CLINC 150 is an intent classification dataset with 150 intentions and 100 sentences per intention~\citep{larson2019,clinc150_570} (CC BY 4.0 license). We embedded sentences with the BART~\citep{Lewis2019-kc} base-sized pretrained model, then embedded them into an 8-dimensional PC embedding. The label for each intent class was embedded with the fine-tuned MINILM~\cite{10.5555/3495724.3496209}\footnote{all-MiniLM-L6-v2, \url{https://huggingface.co/sentence-transformers/all-MiniLM-L6-v2} under Apache-2.0 license} text encoder into 64 latent dimensions with \verb+truncate_dim=64+, then scaled to have unit variance per dimension.

\subsection{Conditional Monge gap}
\label{app:exp_cmonge}
\subsubsection{Model and training}
We adapted the source code of \citet{Driessen2025-vw}\footnote{Source code at \url{https://github.com/AI4SCR/conditional-monge-gap} under MIT license, accessed Nov 1st, 2025}. For each train condition, 1/5 of the samples are held out as validation samples. We train the model for 100 epochs, with one condition per batch, selecting 256 random samples for the condition with replacement. The network was trained with the Adam optimizer with weight decay of $10^{-5}$. Monge gap regularization parameter was set to $1.0$. OT regularization parameter was set to $\epsilon\sigma$ with $\epsilon=0.05$ for Sinkhorn divergence and $\epsilon=0.1$ for Monge gap, where $\sigma$ is the standard deviation of all values in the cost matrix for OT matching. The learning rate was decayed with a cosine decay schedule with constant decay per step until 50 epochs, down to $1/10$ of the initial learning rate. Parameters of the epoch with the least validation loss were selected for test condition evaluation and kernel calculation.

The model separately encodes condition and sample, and concatenated results are fed to a shared MLP to predict the transported particle:\begin{equation}
    T_\theta(x, y) = g_\theta(\text{concat}(f_\theta^{\text{cond}}(x), f_\theta^{\text{sample}}(y)))
\end{equation}
where $f_\theta^{\text{cond}}$, $f_\theta^{\text{sample}}$, $g_\theta$ are MLPs with GELU activation layer~\citep{Hendrycks2016-ji} and $f_\theta^{\text{cond}}$ and $f_\theta^{\text{sample}}$ has last layer activation.

Model configuration and learning rate combination are selected based on the lowest Sinkhorn divergence on the validation conditions from the 5-fold CV of $n$ conditions, where $n=10$ is the minimum initial batch size. Hyperparameter sweep was across 3 model configurations and 3 learning rates $(0.01, 0.001, 0.0001)$.
\begin{table}[h]
\centering
\caption{Conditional Monge Gap model configurations}
\label{tab:model_configs}
\begin{tabular}{lccc}
\toprule
Config & \texttt{$f_\theta^{\text{sample}}$} & \texttt{$f_\theta^{\text{cond}}$} & \texttt{$g_\theta$} \\
\midrule
1  & $[64,64]$    & $[64,64]$    & $[128,128,128,32]$  \\
2 & $[128,128]$  & $[128,128]$  & $[256,256,256,64]$  \\
3 & $[256,256]$  & $[256,256]$  & $[512,512,512,64]$  \\
\bottomrule
\end{tabular}
\end{table}
\begin{table}[h]
\centering
\caption{Selected model configurations for Conditional Monge map datasets}
\label{tab:model_configs_chose}
\begin{tabular}{lccc}
\toprule
Dataset & Model configuration & Learning rate \\
\midrule
K562-Chemical  & 3    & 0.0001   \\
BMDM-Genetic & 3  & 0.001   \\
\bottomrule
\end{tabular}
\end{table}

Test conditions are evaluated with 3 random samples with replacement of 256 samples per condition, both from the source and target samples.

$M=4096$ samples were used to obtain Monte Carlo estimates of kernels. Active learning loops of iterations 0-5 over 5 folds are run on CPU with 100GB of memory within 48 hours on an internal cluster.

\subsubsection{Datasets}
\paragraph{K562-Chemical} K562 cells and MCF7 cells from the Sciplex3 dataset~\citep{Piran2023,Srivatsan2020-vn} (CC BY 4.0 license) are separately embedded into scVI~\citep{Lopez2018-mv} 10-dimensional embeddings with default parameters, and then the mean response embedding in MCF7 is used as the condition covariates. 436 kinase-dose conditions are selected that have at least 80 cells profiled in K562 and are present in MCF7, yielding 64,971 cells, with 3,359 unperturbed cells (80-255 per condition; median 138).

\paragraph{BMDM-Genetic} 
We take the subset of Perturb-seq outcomes that targeted 514 genes in mouse bone marrow-derived macrophages (BMDMs) from \citet{Kudo2026.05.27.728345}. 3000 highly variable genes are selected, and a 50-dimensional salient representation~\citep{Weinberger2022-eg} is obtained. The imaging outcome of the optical pooled screen~\citep{Feldman2019-vb} with matched perturbations was processed and embedded into 70 principal components and pseudo-bulked per perturbation. The top 198 perturbations that diverge the most from the control cells with the largest Wasserstein distance in the 50-dimensional latent space are selected for the experiment in this paper, leading to 73,161 cells with 25,952 unperturbed cells and 256-1,963 cells for each of the 198 selected perturbations. The response of all cells, including the unperturbed cells, was projected onto the first 10-dimensional principal components. Further details of the dataset and preprocessing are described in the original manuscript~\citet{Kudo2026.05.27.728345}\footnote{Public repository appears at \url{https://github.com/Genentech/MoCAVI/blob/main/demos/tutorial_mocavi_kinase_screen.ipynb}}.

\clearpage

\section{Additional Results}
\label{app:additional_results}
\subsection{Empirical analysis of gap in Theorem~\ref{thm:sensitivitybound}}
\label{app:numerical_gap_analysis}
We assessed how loose the upper bound becomes as the dimension grows by analyzing the transport of $d$-dimensional Gaussian distributions. We compute the analytical kernel as described in Appendix~\ref{app:affine-kernel}. We compute the Fréchet variance by obtaining the barycenter of the posterior distribution with randomly sampled parameters $\theta = \{m, b, \ell, k\}$ with $\theta \sim \mathcal{N}(\hat\theta, \sigma^2 I_p)$. 

\paragraph{Barycenter computation.}
  Barycenters of 
  $M=1000$ random parameter samples are obtained with a fixed-point method in \citet{Alvarez-Esteban2015-sl}, until the relative Frobenius norm $F^\mathrm{rel} = \|\Sigma^{(k+1)} - \Sigma^{(k)}\|_F / \|\Sigma^{(k)}\|_F$ converges ($F^\mathrm{rel} \leq 1e^{-10}$). 
  
  For a condition $x$, write $\Sigma_i$, $i=1,\dots, M$, for the covariances of the $M$ jittered pushforward Gaussians at that condition. The Fr\'echet mean
  $(\mu_F,\Sigma_F)$ is computed as follows.

  The mean component is closed form,
  \begin{equation}
    \mu_F \;=\; \frac{1}{M}\sum_{i=1}^{M}\mu_i .
  \end{equation}

  The covariance is the \'Alvarez-Esteban fixed point, initialized at the
  \emph{arithmetic} mean, and iterated as
  \begin{align}
    \Sigma_F^{(0)} &= \frac{1}{M}\sum_{i=1}^{M}\Sigma_i, \\
    \bar{T}^{(k)}  &= \frac{1}{M}\sum_{i=1}^{M}
       \big(\Sigma_F^{(k)}\big)^{-1/2}
       \Big[\big(\Sigma_F^{(k)}\big)^{1/2}\,\Sigma_i\,\big(\Sigma_F^{(k)}\big)^{1/2}\Big]^{1/2}
       \big(\Sigma_F^{(k)}\big)^{-1/2}, \\
    \Sigma_F^{(k+1)} &= \bar{T}^{(k)}\,\Sigma_F^{(k)}\,\bar{T}^{(k)} .
  \end{align}

  Every matrix square root is evaluated spectrally on the symmetrized argument,
  with eigenvalues clamped before the root is taken:
  \begin{equation}
    S^{1/2} = U\,\mathrm{diag}\!\big(\sqrt{\max(\lambda_j,0)}\big)\,U^{\!\top},
    \qquad
    S^{-1/2} = U\,\mathrm{diag}\!\big(\max(\lambda_j,10^{-10})^{-1/2}\big)\,U^{\!\top},
  \end{equation}
  where $U\,\mathrm{diag}(\lambda)\,U^{\!\top}$ is the eigendecomposition of
  $\tfrac{1}{2}(S+S^{\!\top})$.

  The iteration stops at the first $k$ with
  \begin{equation}
    \frac{\big\|\Sigma_F^{(k+1)}-\Sigma_F^{(k)}\big\|_F}
         {\big\|\Sigma_F^{(k)}\big\|_F} \;<\; \tau ,
    \qquad \tau = 10^{-10},
  \end{equation}
The fixed-point iteration converged within 40 iterations across all settings. The norms are Frobenius norms taken over the \emph{stacked} tensor of all conditions, so the criterion is a single pooled residual rather than one test per condition. The accepted iterate is symmetrized and factorized with a small jitter,
  \begin{equation}
    \Sigma_F \leftarrow \tfrac{1}{2}\big(\Sigma_F+\Sigma_F^{\!\top}\big),
    \qquad
    L_F L_F^{\!\top} = \Sigma_F + 10^{-14} I .
  \end{equation}

For each dimension $d$, we calculate $\mathrm{Tr}(G_{\hat\theta}(x))$ and $V_F^{\sigma^2}(x)$ per condition. 

We calculate Fréchet variance from a Monte-Carlo sampling over $M = 1000$ draws of the last-layer perturbation. A single set of $M$ draws is shared across all conditions and across $\sigma \in \{3\!\times\!10^{-4}, 10^{-3}, 3\!\times\!10^{-3}, 10^{-2}\}$ (common random numbers), and is redrawn independently for each of the three seeds of $\sigma$.
For each $\sigma$ and $x$, we independently fit the ordinary least squares regression
\begin{equation}
    V^{\sigma^{2}}_F(x)/\sigma^{2} = a(x) + b(x)\,\sigma^{2}
\end{equation}
and report $\mathrm{Tr}(G_{\hat{\theta}}(x))/\hat{a}(x) - 1$.

We find that the leading-order of the bound grows from $1$ to $1.7$ fold of the Fréchet variance as dimension $d$ grows from $1$ to $50$ (Figure~\ref{fig:bound_growth}).

\begin{figure}[ht]
  \centering
  \includegraphics[width=0.5\textwidth]{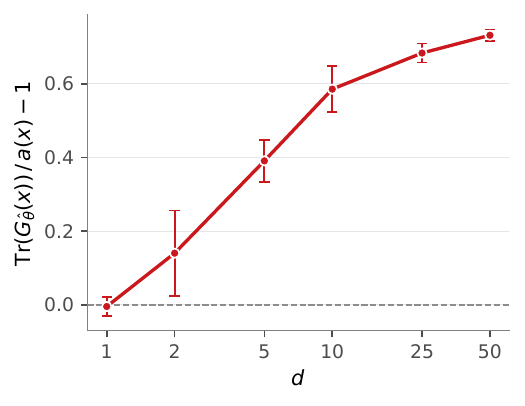}
    \caption{Per-dimension gap in the leading order. Bars are one standard deviation over conditions ($x$).}
    \label{fig:bound_growth}
\end{figure}

\subsection{Multi-dimensional Gaussian affine transport}
\label{app:gaussian_nd}

To test the impact of bound looseness in higher dimensions, we examined active learning performance in the affine transport of higher-dimensional Gaussians in $d \in \{5, 10, 25, 50\}$.

\paragraph{Mean field.}
For $k=0,\dots,d-1$,
\[
\mu_k(x) = \alpha\, e^{-\beta x_1}\cos\Bigl(\frac{2\pi k}{d} + \pi x_2\Bigr),
\qquad \alpha=2,\ \beta=6 .
\]
All $d$ coordinates are driven by the same two scalars, so the mean field stays intrinsically two-dimensional at every $d$.

\paragraph{Covariance field.}
The covariance is a smooth, full-rank random field built from $r=d$ random Fourier
features:
\[
\Sigma(x) = s_d^2\bigl(A(x)A(x)^{\top} + \varepsilon I_d\bigr),\qquad
A(x) = I_d + \sum_{j=1}^{d} f_j(x)\,B_j,
\]
\[
f_j(x) = \sqrt{2/d}\,\cos\bigl(\omega_j^{\top}x + \phi_j\bigr),
\]
with $\varepsilon=10^{-3}$ and
\[
(B_j)_{ab}\overset{\text{iid}}{\sim}\mathcal{N}(0,1/d),\qquad
\omega_j\sim\mathcal{N}\bigl(0,(2\pi\kappa)^2 I_2\bigr),\qquad
\phi_j\sim\mathcal{U}[0,2\pi),\qquad \kappa=1 .
\]
The parameters $\{B_j,\omega_j,\phi_j\}_{j=1}^{d}$ are drawn once per $d$, in that order,
from a dedicated random generator seeded with $d$. They are independent of the experiment
seed, so every seed sees the same ground-truth field. The scale $s_d$ is calibrated so that
the average per-coordinate variance is $\sigma_0^2=1$. Concretely, we draw
$M=2000$ calibration conditions $x^{(m)}\sim\mathcal{U}[0,1]^2$ from a generator with a
fixed seed and set
\[
s_d^2 = \sigma_0^2 \Big/ \Bigl(\frac{1}{Md}\sum_{m=1}^{M}\operatorname{tr}\bigl(A(x^{(m)})A(x^{(m)})^{\top}+\varepsilon I_d\bigr)\Bigr).
\]
With this normalization, the mean and covariance contribute comparably to the variability
of the data at every $d$: the covariance accounts for 41\%, 55\%, 54\% and 51\% of the
pool's Bures--Wasserstein variability at $d=5,10,25,50$. The condition number of
$\Sigma(x)$ grows with $d$, reaching about $3\times10^{3}$ at $d=50$.

\paragraph{Pool, test set and initial set.}
Let $n_{\text{pool}}=\max(200,40d)$, $n_{\text{test}}=\max(100,40d)$ and
$m=\operatorname{round}(\sqrt{n_{\text{pool}}})$.
\begin{itemize}
\item \emph{Pool.} Take the one-dimensional grid $t_i = i/(m-1)$, $i=0,\dots,m-1$, and
  perturb each node independently by $\mathcal{U}(-0.05/m,\,0.05/m)$. The pool is the
  Cartesian product of this perturbed grid with itself ($m^2$ conditions). If $m^2$
  exceeds $n_{\text{pool}}$, $n_{\text{pool}}$ of them are kept uniformly at random. The
  perturbation and the subsample depend on the experiment seed.
\item \emph{Test set.} A fixed, unperturbed $m_t\times m_t$ grid with
  $m_t=\operatorname{round}(\sqrt{n_{\text{test}}})$ and nodes equally spaced on
  $[\tfrac{1}{2m},\,1-\tfrac{1}{2m}]$, i.e.\ offset from the pool by half a cell. It is the
  same for every seed.
\item \emph{Initial labelled set.} $b_0=3d$ pool conditions drawn uniformly at random
  without replacement, depending on the experiment seed.
\end{itemize}
The resulting sizes are:

\begin{center}
\begin{tabular}{lcccc}
\toprule
$d$ & 5 & 10 & 25 & 50 \\
\midrule
pool & 196 & 400 & 1000 & 2000 \\
test & 196 & 400 & 1024 & 2025 \\
initial $b_0$ & 15 & 30 & 75 & 150 \\
\bottomrule
\end{tabular}
\end{center}

\paragraph{Model, training and selection.}
The model has the same two-network structure as in 2D: a mean network with $d$ outputs and
a Cholesky network with $d(d+1)/2$ outputs (softplus on the $d$ diagonal entries). Both are
ReLU MLPs with 2 hidden layers of width 64 at $d=5$ and 3 hidden layers of width 256 for
$d\ge10$. In each round, the networks are re-initialized and trained full-batch on the
labeled conditions with the Bures--Wasserstein loss
$\frac{1}{|S|}\sum_{x\in S}W_2^2\bigl(\mathcal{N}(\mu(x),\Sigma(x)),\mathcal{N}(\hat\mu(x),\hat\Sigma(x))\bigr)$.
Training uses Adam (learning rate $10^{-3}$, weight decay $10^{-2}$, gradient-norm clipping
at 1) for at most 2000 steps, with early stopping allowed after step 1000 once the training
loss has not improved for 10 consecutive steps. Each $d$ is run with 50 seeds, which set the pool perturbation
and subsample, the initial set, the network initializations and the random baseline.
Within a seed, all methods share the pool, test set, and initial set.
\paragraph{Results}
 With $B_0=3d, B=4$, we see the advantage of tNTK over NTK up to $d = 50$ (Table~\ref{tab:gaussian_nd_w2_by_dim}).

\begin{table}[h!]
\centering
\caption{$d$-dimensional Gaussian benchmark with a full-rank covariance field: test $W_2^2$ (total) after the final active-learning round, starting from $b_0=3d$ labeled conditions and acquiring 5 rounds of $b=4$; mean $\pm$ standard error over 50 seeds; lower is better, best per column in bold.}
\label{tab:gaussian_nd_w2_by_dim}
\begin{tabular}{lcccc}
\toprule
Method & $d=5$ & $d=10$ & $d=25$ & $d=50$ \\
\midrule
tNTK & $\mathbf{0.1164} \pm 0.0014$ & $\mathbf{0.7457} \pm 0.0163$ & $\mathbf{1.1017} \pm 0.0169$ & $\mathbf{1.3109} \pm 0.0123$ \\
NTK & $0.1469 \pm 0.0028$ & $0.7942 \pm 0.0141$ & $1.1346 \pm 0.0157$ & $1.4449 \pm 0.0141$ \\
Random & $0.2082 \pm 0.0131$ & $0.7763 \pm 0.0287$ & $1.1150 \pm 0.0222$ & $1.4214 \pm 0.0163$ \\
\bottomrule
\end{tabular}
\end{table}

\subsection{Results on tNTK computation}
\label{app:tntk_computation_results}
In this section, we report compute time and the Monte Carlo stability of the tNTK.
\subsubsection{Kernel compute time}
\label{app:tntk_compute_time}
We report the compute time for the tNTK calculation for the conditional Monge gap and flow matching models reported in the main text in Table~\ref{tab:kernel_time_tntk}.
\begin{table}[h!]
\centering
\footnotesize
\setlength{\tabcolsep}{4pt}
\caption{Wall-clock cost of the tNTK kernel in the reported active-learning runs. Per AL iteration $\geq 1$: kernel construction time (median, with 10th--90th percentile), selection time and model-fit time (medians, seconds), and the kernel's share of the iteration, $t_{\text{kernel}} / (t_{\text{kernel}} + t_{\text{select}} + t_{\text{fit}})$ (median, \%). $N$ is the number of source samples the kernel averages over (conditional Monge: all control cells; CellFlow: fresh $\mathcal{N}(0, I)$ draws). $n$ is the number of AL iterations measured (fit and share use only iterations that trained the model, excluding ones restored from a checkpoint). Times come from the production runs' own instrumentation on shared GPU nodes, including JIT compilation on the first iteration.}
\label{tab:kernel_time_tntk}
\begin{tabular}{llrrrrrr}
\toprule
Dataset & $B_0 + B$ & $N$ & Kernel (s) & Select (s) & Fit (s) & Kernel (\%) & $n$ \\
\midrule
\multicolumn{8}{l}{\emph{Conditional Monge}} \\
\midrule
\multirow{6}{*}{K562-Chemical} & 10 + 1 & 3{,}359 & 16.4 [11.6, 31.3] & 0.2 & 1686 & 1.0 & 247 \\
 & 10 + 5 & 3{,}359 & 21.1 [12.5, 29.9] & 1.5 & 2380 & 0.8 & 177 \\
 & 20 + 5 & 3{,}359 & 15.7 [12.5, 28.2] & 1.3 & 2747 & 0.6 & 153 \\
 & 20 + 10 & 3{,}359 & 37.6 [34.2, 45.1] & 0.6 & 4320 & 0.9 & 250 \\
 & 30 + 5 & 3{,}359 & 45.7 [36.8, 51.8] & 0.3 & 4151 & 1.1 & 250 \\
 & 30 + 10 & 3{,}359 & 44.0 [35.4, 51.6] & 0.7 & 5192 & 0.7 & 250 \\
\cmidrule(lr){1-8}
\multirow{6}{*}{BMDM-Genetic} & 10 + 1 & 25{,}952 & 167.1 [120.3, 199.1] & 1.1 & 1678 & 8.5 & 250 \\
 & 10 + 5 & 25{,}952 & 160.8 [117.3, 222.1] & 6.3 & 3548 & 4.3 & 250 \\
 & 20 + 5 & 25{,}952 & 129.1 [97.5, 179.2] & 6.1 & 4160 & 3.0 & 250 \\
 & 20 + 10 & 25{,}952 & 193.5 [150.7, 232.9] & 0.6 & 4699 & 3.8 & 250 \\
 & 30 + 5 & 25{,}952 & 213.7 [186.7, 251.9] & 0.3 & 4593 & 4.4 & 250 \\
 & 30 + 10 & 25{,}952 & 253.7 [146.2, 361.3] & 0.8 & 6209 & 4.0 & 272 \\
\midrule
\multicolumn{8}{l}{\emph{CellFlow}} \\
\midrule
\multirow{3}{*}{Radial flow} & 5 + 1 & 4{,}096 & 30.9 [28.3, 39.6] & 1.3 & 299 & 9.5 & 150 \\
 & 5 + 5 & 4{,}096 & 31.5 [28.4, 42.1] & 7.7 & 380 & 7.6 & 120 \\
 & 10 + 5 & 4{,}096 & 31.8 [27.8, 45.0] & 6.2 & 359 & 8.2 & 150 \\
\cmidrule(lr){1-8}
\multirow{3}{*}{Van der Pol oscillator} & 5 + 1 & 4{,}096 & 31.6 [26.2, 44.8] & 1.2 & 229 & 13.1 & 150 \\
 & 5 + 5 & 4{,}096 & 30.7 [25.5, 42.9] & 6.4 & 278 & 10.1 & 150 \\
 & 10 + 5 & 4{,}096 & 27.7 [24.8, 33.8] & 6.6 & 255 & 10.1 & 150 \\
\cmidrule(lr){1-8}
\multirow{3}{*}{CLINC 150} & 5 + 1 & 4{,}096 & 40.9 [34.6, 49.1] & 1.5 & 922 & 4.1 & 150 \\
 & 5 + 5 & 4{,}096 & 36.8 [30.1, 51.5] & 7.6 & 1338 & 2.7 & 150 \\
 & 10 + 5 & 4{,}096 & 39.2 [29.1, 59.6] & 6.2 & 1725 & 2.2 & 196 \\
\bottomrule
\end{tabular}
\end{table}

\subsubsection{Kernel stability to the number of Monte Carlo samples}
We tested the Monte Carlo stability of tNTK with varying numbers of samples. We measured the stability metric of the relative Frobenius error ($\|K^\mathrm{full} - K^\mathrm{downsampled}\|_F / \|K^\mathrm{full}\|_F$) and overlap in the selected samples over 20 independent draws and report the mean in Table~\ref{tab:mc_stability_tntk}. Specifically, among the reported runs in the main paper, the MC stability is measured from the fold 0, seed 0, AL iteration 1 model. For conditional Monge gap, the reference uses all control cells (K562 3{,}359, BMDM 25{,}952), which is also what the reported runs use. Draws are without replacement, so $N \geq 3{,}359$ is not available for K562 (--). For conditional flow matching, the source samples are fresh $\mathcal{N}(0, I)$ draws, the reference uses $N = 65{,}536$, and the reported runs use $N = 4{,}096$. 

\begin{table}[h!]
\centering
\small
\setlength{\tabcolsep}{4pt}
\caption{Monte Carlo stability of the tNTK kernel. For each $N$, the kernel is rebuilt from $N$ source samples and compared with a reference kernel. \emph{Err.} is the relative Frobenius error (\%). \emph{Selection \%} is the percentage of conditions in the selected batch of size $B$ that also appear in the batch selected with the reference kernel (out of $B$). Mean over 20 independent draws.}
\label{tab:mc_stability_tntk}
\begin{tabular}{llrrrrrrrr}
\toprule
 & & \multicolumn{2}{c}{$N = 256$} & \multicolumn{2}{c}{$N = 1{,}024$} & \multicolumn{2}{c}{$N = 4{,}096$} & \multicolumn{2}{c}{$N = 16{,}384$} \\
\cmidrule(lr){3-4}\cmidrule(lr){5-6}\cmidrule(lr){7-8}\cmidrule(lr){9-10}
Dataset & $B_0 + B$ & Err. & Selection \% & Err. & Selection \% & Err. & Selection \% & Err. & Selection \% \\
\midrule
\multicolumn{10}{l}{\emph{Conditional Monge}} \\
\midrule
\multirow{6}{*}{K562-Chemical} & 10 + 1 & 2.52 & 100.0 & 1.16 & 100.0 & -- & -- & -- & -- \\
 & 10 + 5 & 2.52 & 70.0 & 1.16 & 82.0 & -- & -- & -- & -- \\
 & 20 + 5 & 4.04 & 100.0 & 1.56 & 100.0 & -- & -- & -- & -- \\
 & 20 + 10 & 4.04 & 99.5 & 1.56 & 99.5 & -- & -- & -- & -- \\
 & 30 + 5 & 4.84 & 97.0 & 2.01 & 100.0 & -- & -- & -- & -- \\
 & 30 + 10 & 4.84 & 96.0 & 2.01 & 100.0 & -- & -- & -- & -- \\
\cmidrule(lr){1-10}
\multirow{6}{*}{BMDM-Genetic} & 10 + 1 & 2.54 & 50.0 & 1.25 & 75.0 & 0.44 & 100.0 & 0.17 & 100.0 \\
 & 10 + 5 & 2.54 & 33.0 & 1.25 & 73.0 & 0.44 & 97.0 & 0.17 & 100.0 \\
 & 20 + 5 & 3.18 & 90.0 & 1.88 & 98.0 & 0.72 & 100.0 & 0.27 & 100.0 \\
 & 20 + 10 & 3.18 & 86.5 & 1.88 & 91.5 & 0.72 & 100.0 & 0.27 & 100.0 \\
 & 30 + 5 & 3.38 & 86.0 & 1.80 & 99.0 & 0.94 & 100.0 & 0.29 & 100.0 \\
 & 30 + 10 & 3.38 & 88.0 & 1.80 & 94.5 & 0.94 & 99.5 & 0.29 & 100.0 \\
\midrule
\multicolumn{10}{l}{\emph{CellFlow}} \\
\midrule
\multirow{3}{*}{Radial flow} & 5 + 1 & 2.53 & 30.0 & 1.90 & 20.0 & 1.06 & 45.0 & 0.71 & 30.0 \\
 & 5 + 5 & 2.68 & 78.0 & 2.02 & 83.0 & 1.24 & 80.0 & 1.03 & 84.0 \\
 & 10 + 5 & 3.68 & 85.0 & 2.52 & 85.0 & 1.55 & 82.0 & 0.87 & 80.0 \\
\cmidrule(lr){1-10}
\multirow{3}{*}{Van der Pol oscillator} & 5 + 1 & 3.62 & 100.0 & 2.72 & 100.0 & 1.47 & 100.0 & 0.99 & 100.0 \\
 & 5 + 5 & 3.62 & 87.0 & 2.71 & 92.0 & 1.47 & 98.0 & 0.99 & 97.0 \\
 & 10 + 5 & 3.50 & 92.0 & 2.44 & 89.0 & 1.44 & 89.0 & 0.87 & 88.0 \\
\cmidrule(lr){1-10}
\multirow{3}{*}{CLINC 150} & 5 + 1 & 1.82 & 100.0 & 1.13 & 100.0 & 0.55 & 100.0 & 0.27 & 100.0 \\
 & 5 + 5 & 1.85 & 74.0 & 1.16 & 87.0 & 0.56 & 92.0 & 0.27 & 98.0 \\
 & 10 + 5 & 1.58 & 71.0 & 0.90 & 91.0 & 0.41 & 100.0 & 0.26 & 100.0 \\
\bottomrule
\end{tabular}
\end{table}
\clearpage

\subsection{Additional baselines}
\label{app:additional_baselines}
We present additional baselines for the results presented in the main text. 
For affine transport of 2D Gaussian distributions, we also benchmarked against BatchBALD (NTK with greedy maximum determinant selection) and BADGE-style (NTK with k-means++ selection) adapted for the regression as in \citet{Holzmuller2022-do} (Table~\ref{tab:gaussian2d_ntk_baselines_iter_mean}). tNTK, with minimization of total posterior variance selection, outperforms all NTK-based baselines. 

\begin{table}[h!]
\centering
\caption{2D Gaussian: NTK-based batch-selection baselines. Decrease of $W_2^2$ relative to the initial training set, averaged over AL iterations 1--5 (units $10^{-3}$; more negative is better), mean $\pm$ s.e.m.\ over 30 seeds.}
\label{tab:gaussian2d_ntk_baselines_iter_mean}
\begin{tabular}{lc}
\toprule
Method & $\Delta W_2^2$ ($\downarrow$) \\
\midrule
tNTK & $\mathbf{-431.58} \pm 26.68$ \\
NTK (BAIT; NTK + total posterior variance minimization)& $-364.15 \pm 26.59$ \\
BatchBALD (NTK + maximum determinant) & $-360.94 \pm 26.30$ \\
BADGE-style (NTK + k-means++) & $-405.88 \pm 27.09$ \\
Random & $-363.93 \pm 26.09$ \\
\bottomrule
\end{tabular}
\end{table}

We also benchmark tNTK in a larger-batch regime with LCMD selection, which \citet{Holzmuller2022-do} has shown to perform the best in a large-batch scenario of $B_0 = 5$, $B = 20$ with 1 active learning iteration. In this scenario, the LCMD selection outperformed total posterior variance minimization, while tNTK outperformed NTK (Table~\ref{tab:gaussian2d_lcmd_vs_bait}). The result shows the importance of choosing the right selection algorithm in various active learning settings, as has been shown before~\citep{NEURIPS2023_1ed4723f}. We defer the investigation of the synergy between the tNTK and selection algorithms to future studies.

\begin{table}[h!]
\centering
\caption{2D Gaussian: LCMD versus minimum total posterior variance (BAIT) as the batch-selection rule. Mean $\pm$ standard error over 50 seeds is shown. The two rules are matched runs (same seeds, initial set, and round-0 model). Random ignores the kernel and the selection rule, so its row is shared.}
\label{tab:gaussian2d_lcmd_vs_bait}
\begin{tabular}{lcc}
\toprule
Kernel & LCMD & Min.\ total posterior variance \\
\midrule
tNTK & $\mathbf{0.0573} \pm 0.0021$ & $0.0788 \pm 0.0026$ \\
NTK & $0.0643 \pm 0.0025$ & $0.2153 \pm 0.0072$ \\
Random & -- & $0.1234 \pm 0.0061$ \\
\bottomrule
\end{tabular}
\end{table}

The experimental settings are identical to those presented in the main text unless otherwise mentioned, and are described in detail in Appendix~\ref{app:exp_2d_gaussian}.

\clearpage
\newpage
\section{Supplementary Figures}
\label{app:supplementary_figs}
\begin{figure}[ht]
  \centering
  \includegraphics[width=\textwidth]{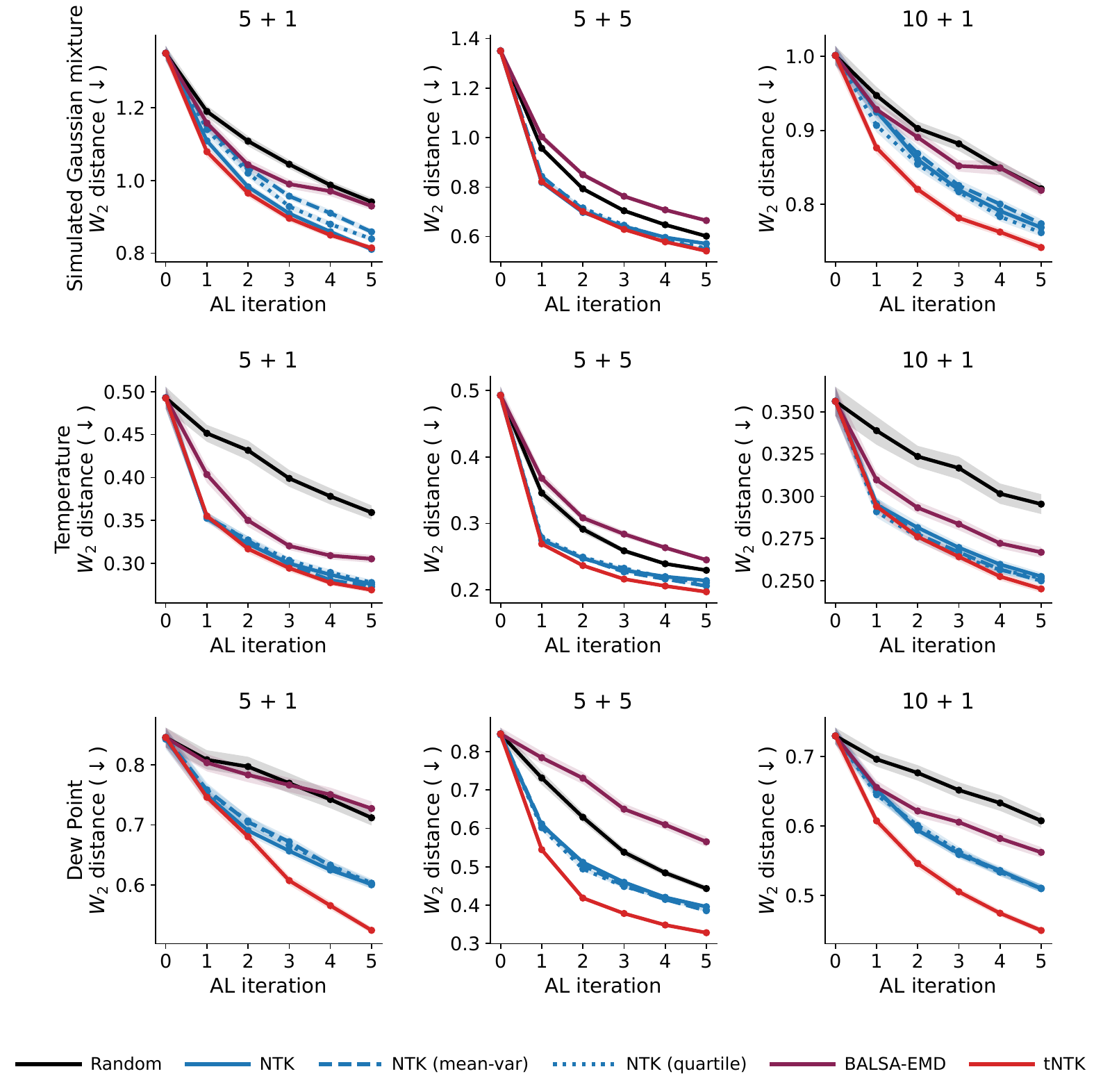}
    \caption{Wasserstein-2 distance over active learning iterations for the conditional quantile regression experiment. Rows correspond to datasets (Simulated Gaussian mixture, Temperature, Dew Point); columns correspond to budget configurations $B_0 + B$ (initial training set size + per-round batch size). Lines show mean across $50$ runs ($5$ folds $\times$ $10$ seeds), with shaded bands showing $\pm$ standard error. Random in black, NTK in blue, tNTK in red.}
  \label{fig:supp_quantile_w2}
\end{figure}

\begin{figure}[ht]
  \centering
  \includegraphics[width=\textwidth]{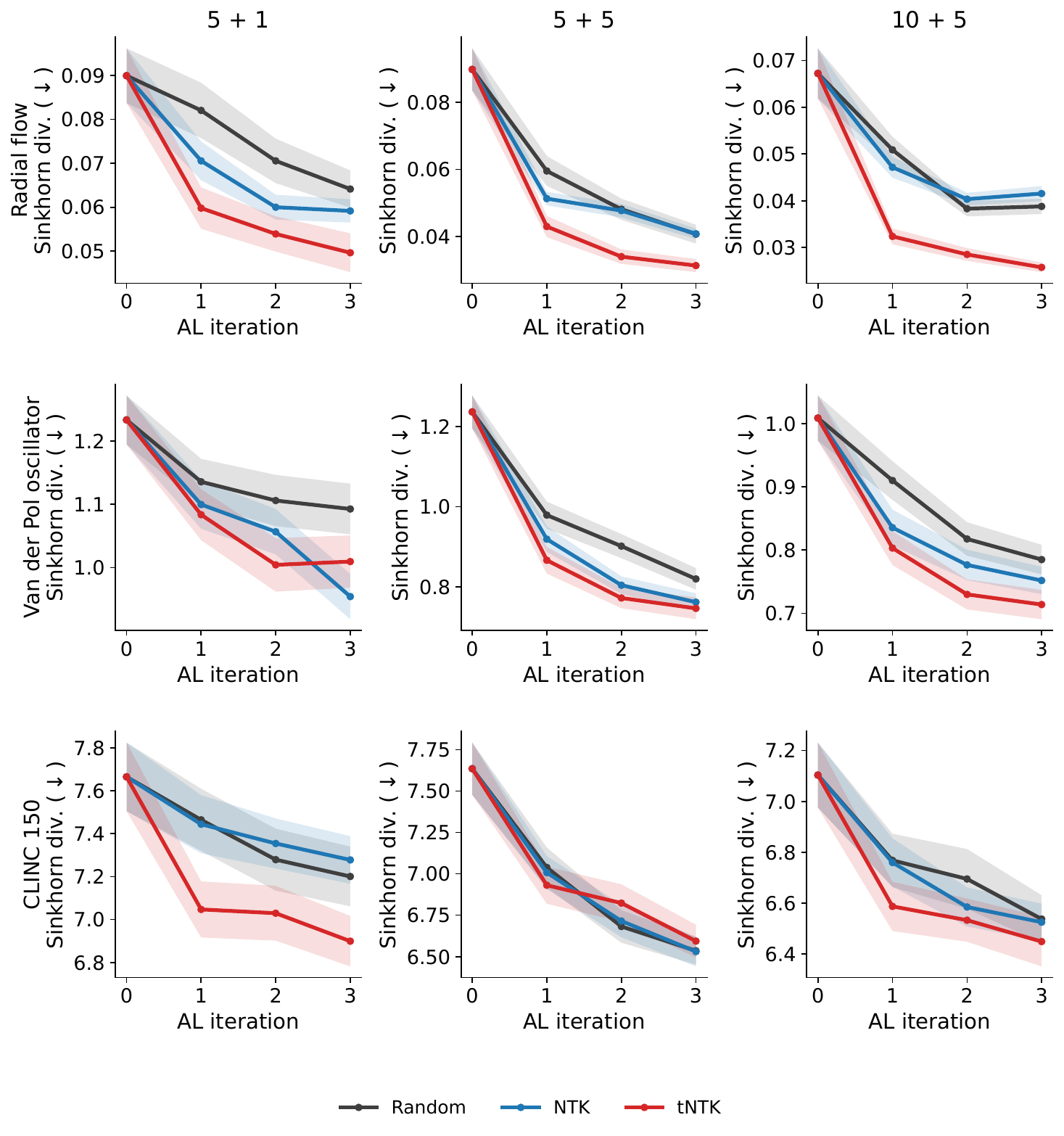}
\caption{Sinkhorn divergence over active learning iterations for conditional flow matching. Rows correspond to datasets (radial flow, Van der Pol oscillator, CLINC150); columns correspond to budget configurations $B_0 + B$ (initial training set size + per-round batch size). Lines show mean across $50$ runs ($5$ folds $\times$ $10$ seeds), with shaded bands showing $\pm$ standard error. Random in black, NTK in blue, tNTK in red.}
  \label{fig:supp_cflow_sinkhorn}
\end{figure}

\begin{figure}[ht]
  \centering
  \includegraphics[width=\textwidth]{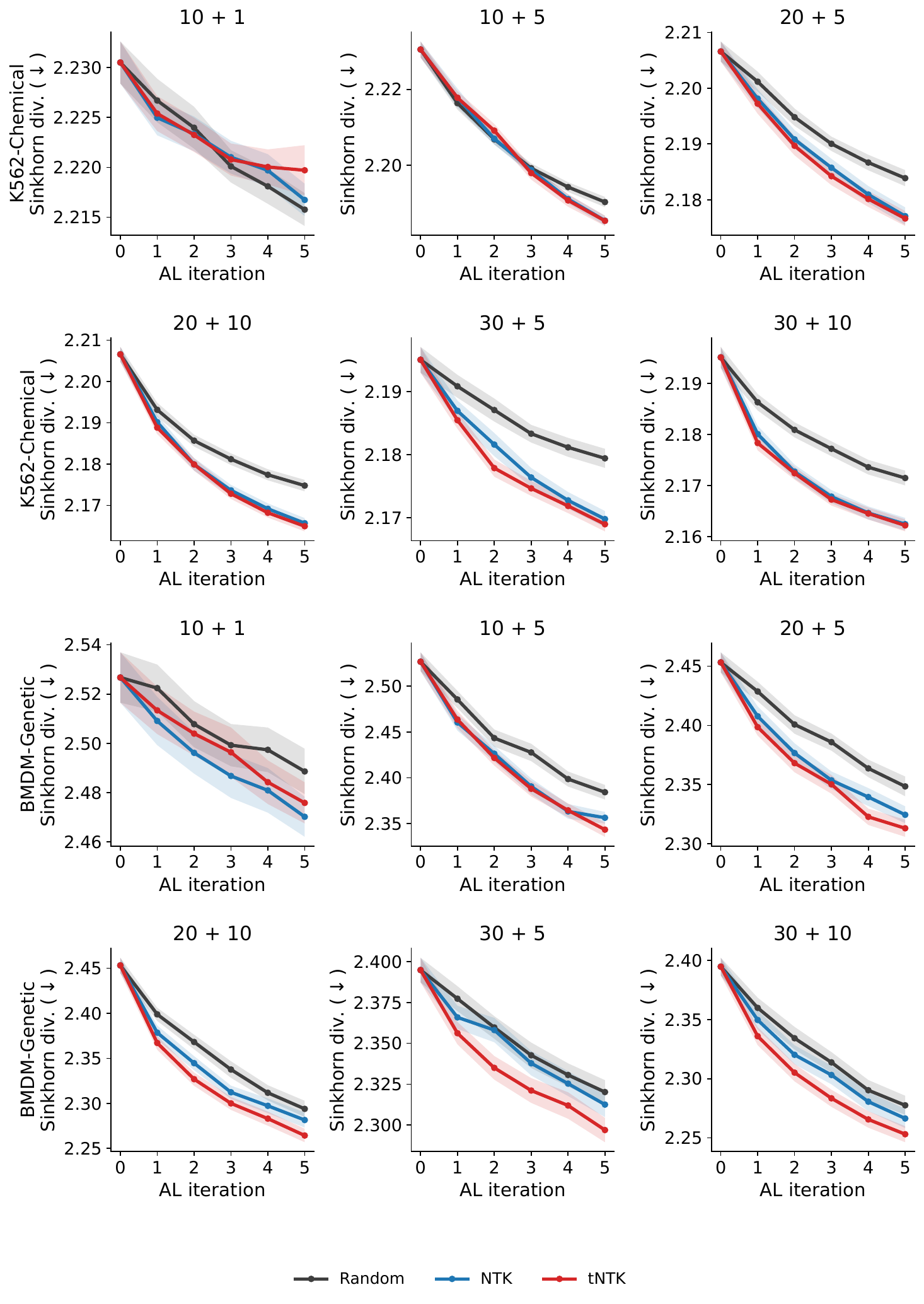}
\caption{Sinkhorn divergence over active learning iterations for the conditional Monge gap. Rows correspond to datasets (K562-Chemical, BMDM-Genetic); columns correspond to budget configurations $B_0 + B$ (initial training set size + per-round batch size). Lines show mean across $50$ runs ($5$ folds $\times$ $10$ seeds), with shaded bands showing $\pm$ standard error. Random in black, NTK in blue, tNTK in red.}

  \label{fig:supp_cmonge_sinkhorn}
\end{figure}

\end{document}